\documentclass[11pt, a4paper]{article}
\usepackage{fullpage} 

\usepackage{graphicx}
\usepackage[ngerman,english]{babel}
\usepackage{amsxtra, amsfonts, amssymb, amstext}
\usepackage{amsthm}
\usepackage{booktabs}
\usepackage{bbm,dsfont}
\usepackage{amsmath}
\usepackage{nicefrac}
\usepackage{xspace}
\usepackage[noadjust]{cite}
\usepackage{url}
\usepackage{graphics,color}
\usepackage[colorlinks,pageanchor=false]{hyperref}
\definecolor{linkblue}{rgb}{0.1,0.1,0.8}
\hypersetup{colorlinks=true,linkcolor=linkblue,filecolor=linkblue,urlcolor=linkblue,citecolor=linkblue}
\usepackage[algo2e,ruled,vlined,linesnumbered]{algorithm2e}
\usepackage{multirow}
\newcommand{\assign}{\leftarrow}

\usepackage{array}
\usepackage{booktabs}
\usepackage{bigdelim}

\newtheorem{theorem}{Theorem}
\newtheorem{lemma}[theorem]{Lemma}

\newtheorem{corollary}[theorem]{Corollary}
\newtheorem{definition}[theorem]{Definition}

\newtheorem{remark}[theorem]{Remark}

\newcommand{\N}{\mathbb{N}}
\newcommand{\R}{\mathbb{R}}

\renewcommand{\epsilon}{\varepsilon}
\newcommand{\eps}{\varepsilon}

\newcommand{\E}{E}

\DeclareMathOperator{\flip}{flip}

\DeclareMathOperator{\opt}{opt}
\newcommand{\Rapp}{\tilde{R}_{\opt}}
\newcommand{\Rappe}{\tilde{R}_{\opt,\eps}}

\DeclareMathOperator{\Bin}{Bin}

\newcommand{\oea}{$(1 + 1)$~EA\xspace}

\newcommand{\OneMax}{\textsc{OneMax}\xspace}
\newcommand{\onemax}{\textsc{OneMax}\xspace}
\newcommand{\OM}{\textsc{Om}\xspace}

\newcommand{\jump}{\textsc{Jump}\xspace}
\newcommand{\plateau}{\textsc{Plateau}\xspace}
\newcommand{\trap}{\textsc{Trap}\xspace}
\newcommand{\needle}{\textsc{Needle}\xspace}

\newcommand{\binval}{\textsc{BinaryValue}\xspace}

\newcommand{\leadingones}{\textsc{LeadingOnes}\xspace}

\allowdisplaybreaks

\begin{document}

\title{Optimal Parameter Choices via Precise Black-Box Analysis\thanks{A preliminary version of this work was presented at the \emph{Genetic and Evolutionary Computation Conference (GECCO)} 2016~\cite{DoerrDY16}.}}
\author{Benjamin Doerr \\ \'Ecole Polytechnique, CNRS\\ LIX - UMR 7161\\ 91120 Palaiseau\\ France      \and
	Carola Doerr \\ Sorbonne Universit\'e, CNRS\\ Laboratoire d'informatique de Paris 6, LIP6\\ 75252 Paris\\ France \and 
	Jing Yang \\ \'Ecole Polytechnique, CNRS\\ LIX - UMR 7161\\ 
	91120 Palaiseau\\ France}




\maketitle

{\sloppy
\begin{abstract}
  It has been observed that some working principles of evolutionary algorithms, in particular, the influence of the parameters, cannot be understood from results on the asymptotic order of the runtime, but only from more precise results. In this work, we complement the emerging topic of precise runtime analysis with a first precise complexity theoretic result. Our vision is that the interplay between algorithm analysis and complexity theory becomes a fruitful tool also for analyses more precise than asymptotic orders of magnitude.
  
As particular result, we prove that the unary unbiased black-box complexity of the OneMax benchmark function class is $n \ln(n) - cn \pm o(n)$ for a constant $c$ which is between $0.2539$ and $0.2665$. This runtime can be achieved with a simple (1+1)-type algorithm using a fitness-dependent mutation strength. When translated into the fixed-budget perspective, our algorithm finds solutions which are roughly 13\% closer to the optimum than those of the best previously known algorithms. To prove our results, we formulate several new versions of the variable drift theorems, which also might be of independent interest.
\end{abstract}



\sloppy{
\section{Introduction}
\label{sec:Intro}

An important goal of the theory of randomized search heuristics (RSHs) is to prove mathematically founded statements about optimal parameter choices for these algorithms. The area of runtime analysis has contributed to this goal with rigorous analyses showing how the runtime of RSHs depends on one or more of their parameters. Unfortunately, due to the inherent difficulty of obtaining mathematically proven performance guarantees for RSHs, the majority of the existing runtime analyses only determine the asymptotic order of magnitude of the runtime (that is, the runtime in big-Oh notation). Naturally, such results usually give recommendations on optimal parameters again precise only up to the asymptotic order of magnitude, which for practical uses is often not precise enough. Only recently, made possible by the advancement of the analytical tools in the last 20 years, a number of results appeared that also make the leading constant precise or even some lower order terms. These \emph{precise runtime analyses} usually allowed much more precise statements about the ideal parameter choice. 

In this work, we shall undertake the first steps to complement precise runtime analysis with equally precise complexity theoretic results. This is inspired by the area of classic algorithms theory, where the interplay between algorithm analysis and complexity theory has led to great advances. As in previous works on complexity theory for evolutionary algorithms, we build on the notion of \emph{black-box complexity} introduced in the seminal work~\cite{DrosteJW06}. In very simple words, the black-box complexity of an optimization problem is the number of fitness evaluations that suffice to find an optimal solution. This number is witnessed by a theoretically best-possible black-box algorithm, which may or may not be an evolutionary algorithm. In the former case, we have a proof that no better evolutionary algorithm exists, in the latter, one may ask the question why the existing evolutionary algorithms fall behind the theoretical optimum. As in classic algorithms, this can be a trigger to improve the existing algorithms. The example of~\cite{DoerrDE15} shows that a careful analysis of the theoretically optimal black-box algorithm can also guide the design of new evolutionary algorithms. 


\subsection{Previous Works on Precise Runtime Analyses}

As said above, the inherent difficulty of proving runtime guarantees for evolutionary algorithms and other RSHs by mathematical means for a long time prohibited runtime results that are more precise than giving the asymptotic order of magnitude. In fact, we note that for many simple optimization problems and very simple evolutionary algorithms, we even do not know the asymptotic order of the runtime. Examples include the runtime of the \oea on the minimum spanning tree problem~\cite{Witt14gecco} and on the single-source shortest path problem with the natural single-criterion fitness function~\cite{manyFOGA09}, or the runtime of the global SEMO algorithm for the multi-objective LeadingOnesTrailingZeros problem~\cite{DoerrKV13}.

Before describing the few known precise runtime results, let us argue that in evolutionary computation precise results are more important than in classic algorithms theory. The performance measure in classic algorithms theory is the number of elementary operations performed in a run of the algorithm. Since the different elementary operations have different execution times and since further these are dependent on the hardware used, this performance measure can never give meaningful results that are more precise than the asymptotic order of magnitude. In evolutionary computation, the performance measure is the number of fitness evaluations. For this implementation-independent measure, results more precise than asymptotic orders are meaningful since we can often expect that each fitness evaluation takes more or less the same time and since we usually assume that the total runtime is dominated by the time spent on fitness evaluations. 


Despite this interest in precise runtime estimates, only few results exist which determine a runtime precisely, that is, show upper and lower bounds that have at least the same leading constant. For classic 
evolutionary algorithms, the following is known.

A very early work on precise runtime analysis, apparently overlooked by many subsequent works, is the very detailed analysis on how the $(1+1)$ evolutionary algorithm (\oea) with general mutation rate $c/n$, $c$ a constant, optimizes the \needle and \onemax functions~\cite{GarnierKS99}. Disregarding here the results on the runtime distributions, this work shows that the \oea with mutation rate $c/n$ finds the optimum of the \needle function in $(1+o(1)) \frac{1}{1-e^{-c}} 2^n$ time. 

For \onemax,
the runtime estimate in~\cite{GarnierKS99} is $(1+o(1)) \frac{e^c}{c} n \ln(n)$, more precisely, $\frac{e^c}{c} n \ln(n) \pm O(n)$. The proof of the latter result uses several deep tools from probability theory, among them Laplace transforms and renewal theory. Note that the main difficulty is the lower bound. For a simple proof of the upper bound $(1+o(1)) e n \ln(n)$ for mutation rate $1/n$, see, e.g.,~\cite{DrosteJW02}. The first proofs of a lower bound of order $(1-o(1)) e n \ln(n)$ using more elementary methods were given, independently, in~\cite{DoerrFW10,Sudholt13}. An improvement to  $e n \ln(n) - O(n)$ was presented in~\cite{DoerrFW11}. Very recently, a very precise analysis of the expected runtime, specifying all lower order terms larger than $\Theta(\log(n)/n)$ with precise leading constants, was given in~\cite{HwangPRTC18}. 

That the runtime bound of $(1+o(1)) \frac{e^c}{c} n \ln(n)$ holds not only for \onemax, but any linear pseudo-Boolean function, was shown in~\cite{Witt13}, ending a long quest for understanding this problem~\cite{DrosteJW02,HeY04,Jagerskupper08,DoerrJW12}.
An extension of the \onemax result to $(1+\lambda)$ EAs was obtained in~\cite{GiessenW15}. The bound of $(1+o(1)) (\frac{e^c}{c} n \ln(n) + n \lambda \ln\ln(\lambda)/2\ln(\lambda))$ fitness evaluations contains the surprising result that the mutation rate is important for small offspring population sizes, but has only a lower-order influence once $\lambda$ is sufficiently large (at least when restricted to mutation rates in the range $\Theta(1/n)$; note that \cite[Lemma~1.2]{DoerrGWY17} indicates that mutation rates of a larger order of magnitude can give asymptotically smaller runtimes for larger values of $\lambda$). Results from~\cite{Dang-NhuDDIN18} imply that runtime of the \oea with mutation rate $1/n$ remains $(1+o(1)) en\ln(n)$ when the EA has to cope with dynamic changes of the optimum moving the optimum by an expected distance of $c\ln\ln(n)/n$, $c$ a sufficiently small constant, independently in each iteration. To see this result, put $t=n$ in Theorem~5 of~\cite{Dang-NhuDDIN18} and note that the probability $p_D$ for a dynamic change always is at most the expected change $E_D$. When calling this a precise runtime result, we tacitly assume that the known lower bound of $(1-o(1))en\ln(n)$ holds in the dynamic setting as well, as it is hard to believe that the \oea should profit from random dynamic moves of the optimum.

In parallel independent work, the precise expected runtime of the \oea on the \leadingones benchmark function was determined in~\cite{BottcherDN10,Sudholt13} (note that~\cite{Sudholt13} is the journal version of a work that appeared at the same conference as~\cite{BottcherDN10}). The work~\cite{Sudholt13} is more general in that it also regards the \oea with Best-of-$\mu$ initialization (instead of just initializing with a random individual), the approach of~\cite{BottcherDN10} has the advantage that it also allows to determines the distribution of the runtime (this was first noted in~\cite{DoerrJWZ13gecco} and was formally proven in~\cite{Doerr18evocop}). The work~\cite{BottcherDN10} also shows that the often recommended mutation rate of $p=1/n$ is not optimal. A runtime smaller by 16\% can be obtained from taking $p = 1.59 / n$ and another 12\% can be gained by using a fitness-dependent mutation rate. For several hyper-heuristics having the choice between the 1-bit flip and the 2-bit-flip operators, precise runtimes have been determined recently~\cite{LissovoiOW17,DoerrLOW18}. These results in particular show that the classic selection hyper-heuristics \emph{simple random}, \emph{random gradient}, \emph{greedy}, and \emph{permutation} all have an inferior runtime to variants of the \emph{random gradient} heuristic which use a chosen operator for a longer period $\tau$ than just until the first time no improvement is found. Since all heuristics regarded have a runtime of asymptotic order $\Theta(n^2)$, this runtime distinction would not have been possible without results making the leading constant precise. 

In~\cite{DoerrLMN17}, the runtime of the \oea on \jump functions was determined precise up to the leading constant. This showed that the optimal mutation rate for jump size $k = O(1)$ is $(1+o(1)) k / n$, significantly larger than the common recommendation $1/n$. Also, this work led to the development of a heavy-tailed mutation operator which gave a uniformly good performance for all jump sizes $k$. 

With very different methods, a precise runtime analysis for \plateau functions was conducted in~\cite{AntipovD18}. The plateau function with radius~$k$ is similar to the jump function with parameter $k$, the difference being that now a plateau of equal fitness of radius $k$ around the optimum is the difficult part of the search space. The precise runtime of the \oea with arbitrary unbiased mutation operator (flipping one bit with at least some constant probability) is $(1+o(1)) \binom{n}{k} p_{1:k}$, where $p_{1:k}$ denotes the probability that the mutation operator flips between one and $k$ bits. This implies again that a larger mutation rate than usual is optimal, namely $\sqrt[k]{k}/n \approx k/en$.

In summary, there are not too many results determining precise runtimes, but the ones we have significantly increased out understanding of how evolutionary algorithms work and what are good parameter values.

%

\subsection{Black-Box Complexity}

The notion of black-box complexity was introduced in~\cite{DrosteJW06} with the goal of establishing a complexity theory for evolutionary algorithms. Complexity theory means that we aim at understanding the difficulty of a problem, that is, how well the best algorithm (from a specified class) can solve this problem. Since it is hard to provide a mathematically sound definition of what an evolutionary algorithm is, black-box complexity regards instead the (larger) class of black-box algorithms, that is, all algorithms which have access to the problem instance only via evaluating solution candidates. Consequently, the black-box complexity of a problem is the smallest number of fitness evaluation such that there is a black-box optimization algorithm solving all instances of the problem with at most this expected number of fitness evaluations.
%
%
%
%

It was observed early that the class of all black-box algorithms is significantly larger than the class of evolutionary algorithms. This led to some unexpectedly low black-box complexities witnessed by highly artificial algorithms. To develop a more suitable complexity theory, but also to study the effect of particular restrictions, several restricted notions of black-box complexity have been suggested, among them unbiased black-box complexity (admitting only algorithms which generate solution candidates from previous solutions in an unbiased fashion)~\cite{LehreW12}, memory restricted black-box complexity (allowing to store only a certain number of solution candidates)~\cite{DrosteJW06,DoerrW12ipl}, and ranking-based black-box complexity (allowing only to compare fitness values, but not to exploit absolute fitness values)~\cite{FournierT11,DoerrW14ranking}. We refer to the survey~\cite{Doerr18BBC} for a detailed discussion of these and further black-box complexity notions.

In this work, we build on the unary unbiased black-box complexity~\cite{LehreW12}, which is the most appropriate model for mutation-based search heuristics. In simple words, a unary unbiased black-box algorithm for the optimization of a pseudo-Boolean function $f : \{0,1\}^n \to \R$ is allowed (i)~to sample search points uniformly at random from the search space $\{0,1\}^n$, and (ii)~to generate new search points from applying unbiased mutation operators to previously found search points. Here \emph{unbiased} means that the operator is invariant under automorphisms of the hypercube $\{0,1\}^n$. In other words, the operator has to treat both the bit-positions and the bit-values in a symmetric fashion. For unary operators (usually called \emph{mutation operators}), we give a simple characterization of the set of unbiased operators in Lemma~\ref{lem:characterizationunary}, roughly saying that an unbiased mutation operator always can be described via first sampling a number $r \in [0..n] := \{0, \dots, n\}$ according to a given probability distribution and then flipping a set of $r$ bits chosen uniformly at random from all $r$-sets of bits. 
Furthermore, in an unbiased algorithm all selection operations may only rely on the observed fitness values, but not on the particular representations of the solutions. 

While we claim that this work is the first significant progress towards a useful precise complexity theory for evolutionary algorithms, we note that our work is not the first to show precise black-box complexity results. In~\cite{DrosteJW06}, the unrestricted black-box complexity of the \needle and \trap function classes were shown to be $(2^n+1)/2$ (Theorem~1 and Proposition~3), that of the \binval class was shown to be $2 - 2^{-n}$ (Theorem~4), and that of the not-permuted \leadingones class was proven to be $(1\pm o(1)) n / 2$ (Theorem~6). In~\cite[Theorem~16]{DoerrKLW13}, it was shown that the unrestricted black-box complexity of the multi-criteria formulation of the single-source shortest path problem is $n-1$, where $n$ denotes the number of vertices of the input graph. In~\cite{BuzdalovDK16}, the unrestricted black-box complexity of extreme \jump functions, that is, with jump size that large that only the middle (for $n$ even) or the two middle (for $n$ odd) Hamming levels are visible was determined to be $n + \Theta(\sqrt n)$. Since all these result regard unrestricted black-box complexities and some find artificially small complexities, we do not feel that these results contribute immediately towards a precise black-box complexity theory that can help  to improve the evolutionary algorithms currently in use.

\subsection{Overview of Our Results}

With the goal of starting a complexity theory targeting precise results, we analyze the unary unbiased black-box complexity of the $\onemax$ benchmark function 
$$\OM: \{0,1\}^n \rightarrow \R, x \mapsto \sum_{i=1}^n x_i.\footnote{The reader not familiar with black-box complexity may wonder that this is a single function, whose optimum is of course known to be the string $(1,\ldots,1)$. However, we are interested in the time needed by a unary unbiased algorithm---which cannot simply query this string but needs to construct it from unary unbiased operators---to find this string. Note further that the performance of any such algorithm is identical on all functions having a fitness landscape isomorphic to that of $\OM$. That is, for every $z\in\{0,1\}^n$ and every unary unbiased algorithm $A$, the performance of $A$ on $f_z:\{0,1\}^n \rightarrow \{0,1,...,n\}, x \mapsto |\{i \in [n]| x_i=z_i\}|$ is identical to its performance on $\OM=f_{(1,\ldots,1)}$.}$$

It is known that the unary unbiased black-box complexity of \onemax is of order $\Theta(n \log n)$~\cite{LehreW12}. The simple \emph{randomized local search} (RLS) heuristic, a hill-climber flipping single random bits, is easily seen to have a runtime of at most $n \ln(n) + \gamma n + \frac 12 \approx n \ln(n) + 0.5772n$ where $\gamma = 0.5772\dots$ is the Euler-Mascheroni constant. This can be shown by a reduction to the coupon collector problem. The precise complexity~\cite{DoerrD16} of this algorithm is 
\[n \ln(n) + (\gamma - \ln(2))n + o(1) \approx n \ln n - 0.1159n.\] 
Using a best-of-$\mu$ initialization rule with a suitably chosen $\mu$, the running time of RLS decreases by an additive $\Theta(\sqrt{n \log n})$ term~\cite{AxelDD15}. Prior to the present work, this was the best unary unbiased algorithm known for \onemax, and thus its time complexity was the best known upper bound for the unary unbiased black-box complexity of this function.

In this work, we show that the unary unbiased black-box complexity of \onemax is $$n \ln(n) - cn \pm o(n)$$ for a  constant $c$ for which we show $0.2539 < c < 0.2665$. We also show how to numerically compute this constant with arbitrary precision. More importantly, our analysis reveals (and needs) a number of interesting structural results which enlarge our general understanding and which might find applications in other algorithm analyses in the future. In particular, we observe the following, which we will discuss in more detail in the following subsections.
\begin{enumerate}
	\item Drift-maximization is near-optimal: On \onemax, any unary unbiased algorithm which in each iteration (by selecting a suitable parent and choosing a suitable mutation operator) maximizes the expected fitness gain over the best-so-far individual (``drift'') has essentially an optimal runtime. More precisely, its runtime exceeds the optimal one (the unary unbiased black-box complexity) by at most $O(n^{2/3} \log^9 n)$. 
	\item There is such a drift-maximizing algorithm which selects in each iteration the best-so-far solution and mutates it by flipping a fixed number of bits. This number depends solely on the fitness of the current-best solution. It is always an odd number and it decreases with increasing fitness.
	\item In the language of fixed-budget computation introduced by Jansen and Zarges~\cite{JansenZ14}, the drift-maximizing algorithm (and thus also any optimal unary unbiased black-box algorithm) computes solutions with expected fitness distance to the optimum roughly $13\%$ smaller than the previous-best algorithms (RLS, RLS with best-of-$\mu$ initialization).
\end{enumerate}

To show these results, we use a number of technical tools which might be suitable for other analyses as well. Among them are simplified (but sightly stronger) variants of the variable drift theorems for discrete search spaces and versions of the lower bound variable drift theorem which can tolerate large progresses if these happen with sufficiently small probability.

\subsection{Maximizing Drift is Near-Optimal}

Evolutionary algorithms build on the idea that iteratively maximizing the fitness is a good approach. This suggests to try to generate the offspring in a way that the expected fitness gain over the best-so-far solution is maximized. Clearly, this is not a successful idea for each and every problem, as easily demonstrated by examples like the \textsc{distance} and the \textsc{trap} functions~\cite{DrosteJW02}, where the fitness leads the algorithm into a local optimum, or the difficult-to-optimize monotonic functions constructed in~\cite{DoerrJSWZ13,LenglerS18,Lengler18}, where the fitness leads to the optimum, but via a prohibitively long trajectory. Still, one might hope that for problems with a good fitness-distance correlation (and \OM has the perfect fitness-distance correlation), maximizing the expected fitness gain is a good approach. This is roughly what we are able to show.

More precisely, we cannot show that maximizing the expected fitness gain leads to the optimal unary unbiased black-box algorithm. 
It turns out this is also not true, even if we restrict ourselves to elitist algorithms, which cannot use tricks like minimizing the fitness and inverting the search point once the all-zero string was found. In fact, the elitist algorithm flipping in each iteration the number of bits that minimizes the expected remaining runtime is different from the drift-maximizing one. For all realistic problem sizes, however, the differences between the expected runtime of our drift maximizer and that of the optimal elitist algorithm are negligibly small~\cite{NathanCarola18}.

What we can prove, however, is that the algorithm which in each iteration takes the best-so-far solution and applies to it the unary unbiased mutation operator maximizing the expected fitness gain has an expected optimization time which exceeds the unary unbiased black-box complexity by at most an additive term of order $O(n^{2/3} \log^9 n)$.

We note that this result, while natural, is quite difficult to obtain and relies on a number of properties particular to this process, in particular, the fact that we have a good structural understanding of the maximal drift.

\subsection{Maximizing the Drift via the Right Fitness-Dependent Mutation Strength}
\label{sec:intromaxdrift}

Once we decided how to choose the parent individual, in principle it is easy to mutate it in such a way that the drift is maximized. Since any unary unbiased mutation operator can be seen as a convex combination of $r$-bit flip operators, $r \in [0..n]$, and since the drift stemming from such an operator is the corresponding convex combination of the drifts stemming from these operators, we only need to determine, depending on the fitness of the parent, a value for $r$ such that flipping $r$ random bits maximizes the fitness gain over the parent. For a concrete value of $n$ and a concrete fitness of the parent, one can compute the best value of $r$ in time $O(n^2)$. 

We need some more mathematical arguments to (i)~obtain a structural understanding of the optimal $r$-value and (ii)~to obtain a runtime estimate valid for all values of $n$. To this aim, we shall first argue that when the fitness distance $d$ of the parent from the optimum is at most $(\frac 12 - \eps) n$ for an arbitrarily small constant $\eps$, then the optimal drift is obtained from flipping a constant number $r$ of bits (Lemma~\ref{lem:constantr}). For any constant $r$, the drift obtained from flipping $r$ bits can be well approximated by a degree $r$ polynomial in the relative fitness distance $d/n$ (Theorem~\ref{thm:approxBA}). By this, we overcome the dependence on $n$, that is, apart from this small approximation error we can, by regarding these polynomials, determine a function $\tilde R_{\opt} : [0,\frac 12 - \eps] \to \N$ such that for all $n \in \N$ and all $d \in [0..(\frac 12 - \eps)n]$ the near-optimal number of bits to flip (that is, optimal apart from the approximation error) is $\tilde R_{\opt}(d/n)$. This $\tilde R_{\opt}$ is decreasing, that is, the closer we are to the optimum, the smaller is the near-optimal number of bits to flip. Interestingly, $\tilde R_{\opt}$ is never even, so the near-optimal number of bits to flip is always odd (the same holds for the truly optimal number of bits, as we also show). These (and some more) structural properties allow to compute numerically the interval in which flipping $r$ bits is near-optimal (for all odd $r$). From these we estimate the runtime  by numerically approximating the integral describing the runtime via the resulting drifts. This gives approximations for both the runtime of our drift-maximizer and the unary unbiased black-box complexity, which are precise apart from an arbitrary small $O(n)$ term.

%
 
We note that previous works have studied drift maximizing variants of RLS and the (1+1)~EA by empirical means. For $n=100$, B\"ack~\cite{Back92} computed the drift-maximizing mutation rates for different $(1+\lambda)$ and $(1,\lambda)$ EAs. Fialho and co-authors considered in~\cite{FialhoCSS08,FialhoCSS09} for $n=1,000$ and $n=10,000$, respectively, $(1+\lambda)$-type RLS-variants which choose between flipping exactly 1, 3, or 5 bits or applying standard bit mutation with mutation rate $p=1/n$. A simple empirical Monte Carlo evaluation is conducted to estimate the average progress obtained by any of these four operators. None of the three mentioned works, however, further investigates the difference between the drift maximizing algorithm and the optimal (i.e., time-minimizing) one.



\subsection{Fixed-Budget Result}

Computing the runtime of our drift-maximizing algorithm, we observe that the fitness-dependent mutation strength gives a smallish-looking improvement of roughly $0.14n$ in the $\Theta(n \log n)$ runtime. However, if we view our result in the fixed-budget perspective~\cite{JansenZ14}, then (after using the Azuma inequality in the martingale version to show sufficient concentration) we see that if we take the expected solution quality after a fixed number (budget) of iterations as performance measure, then our algorithm gives a roughly 13\% smaller fitness distance to the optimum compared to the previous-best algorithm (provided that the budget is at least $0.2675n$). 

\subsection{Methods}

To obtain our results, we use a number of methods which might find applications is other precise runtime and black-box complexity analyses. Among these, we want to highlight here a few new versions of the variable drift theorems. 

A simple, but useful variant of Johannsen's upper bound drift theorem can be obtained for processes on the non-negative integers. Here the expected runtime can be simply written as the sum of the reciprocals of the lower bounds on the drift (Theorem~\ref{dis_U}). This avoids the use of integrals as in Johannsen's formulation~\cite{Johannsen10}. In our application, we profit from the fact that the new theorem is more precise as it avoids the error stemming from approximating the discrete drift via a continuous progress estimate (this error is usually small, but for our purposes the resulting error of order $\Theta(n)$ would be too large).

For the lower bound variable drift theorem from~\cite{DoerrFW11}, besides a similar simplification for processes on the non-negative integers, we add two improvements. The first concerns the requirement that the process may not make too large progress in a single step. While it is clear that some such condition is necessary for the drift result to hold, the strict condition of~\cite{DoerrFW11} is difficult to enforce in evolutionary computation. When, e.g., using standard-bit mutation, large progresses are highly unlikely (simply because the number of bits that flip is small), but large progresses cannot be ruled out completely. For this reason, we devise a variant that can deal with such situations. Our relaxed condition is that large progresses occur only with small probability. We note that a similar drift theorem was given in~\cite[Theorem~2]{GiessenW16}, however, it appears to be more technical than our version. 

A second difficulty with the drift theorem in~\cite{DoerrFW11} (and likewise with the one in~\cite{GiessenW16}) is the fact that a function $h$ has to be given such that the drift from the point $x$ can be bounded from above by $h(c(x))$, where $c$ is the upper bound for the progress possible from $x$. This is significantly less convenient than just finding a bound $h(x)$ for the drift from point $x$ as in the upper bound drift theorem. Therefore we formulate in Theorem~\ref{dis_L'} a variant of the lower bound drift theorem which also only requires an upper bound $h(x)$ for the drift from $x$. While not a deep result, we expect that this version eases the process of finding lower bounds via variable drift in the future. 

\subsection{Connection to Parameter Control}

As discussed in Section~\ref{sec:intromaxdrift} the drift-maximizing algorithm uses mutation strengths that depend on the fitness value of a current-best solution. These drift-maximizing mutation strengths change from flipping half the bits (for search points with fitness $n/2$) to flipping single bits (for search points close to the optimal solution). The algorithm is therefore an example for a state-dependent parameter controlled EA, in the taxonomy presented in~\cite{DoerrD18chapter}. For reasons of space and focus, we do not give an extended introduction to dynamic parameter choices here in this work, but refer the interested reader to the recent book chapter~\cite{DoerrD18chapter} instead. We note, however, that  in~\cite{DoerrDY16PPSN}, a work subsequent to the conference version of the present work~\cite{DoerrDY16}, we have presented a learning-based control mechanism which tracks the drift-maximizing mutation strength so closely that its overall expected optimization time is at most an additive $o(n)$ term worse than that of the drift-maximizer.

\section{Problem Setting and Useful Tools}
\label{sec:preliminaries}

In this section we briefly describe the black-box setting regarded in this work, the unary unbiased model proposed by Lehre and Witt~\cite{LehreW12}. The variation operators that are admissible in this model are characterized in Lemma~\ref{lem:characterizationunary}. We also collect (Section~\ref{sec:drift}) a number of drift theorems that we build upon in our mathematical analysis. In Theorems~\ref{dis_U} to~\ref{dis_L'} we present new versions of the variable drift theorems which we will use.

\subsection{The Unary Unbiased Black-Box Setting}
\label{sec:setting}\label{SEC:SETTING}\label{sec:bb}

The main goal of our work is to determine a precise bound for the unary unbiased black-box complexity of \onemax, the problem of maximizing the function $\OM:\{0,1\}^n \rightarrow \R, x \mapsto \sum_{i=1}^n{x_i}$, which assigns to each bit string the number of ones in it. That is, we aim at identifying a best-possible mutation-based algorithm for this problem. 
The unary unbiased black-box complexity of \onemax is the smallest expected number of function evaluations that any algorithm following the structure of Algorithm~\ref{alg:unbiased} exhibits on this problem before and including the first evaluation of the unique global optimum $(1,\ldots,1)$.\footnote{In the interest of a concise discussion, we do not provide here an extended discussion of black-box complexity. Interested readers are referred to~\cite{Doerr18BBC} for a summary of different black-box complexity models and known results.} In line~9 of Algorithm~\ref{alg:unbiased} a unary unbiased variation operator is asked for. In the context of optimizing pseudo-Boolean functions, a \emph{unary} operator is an algorithm that is build upon a family $(p(\cdot \mid x))_{x \in \{0,1\}^n}$ of probability distributions over $\{0,1\}^n$. For a given input $x$, the operator outputs a new string that it samples from the distribution $p(\cdot \mid x)$. A unary operator is \emph{unbiased} if all members of its underlying family of probability distributions are symmetric with respect to the bit positions $[n]:=\{1,\ldots,n\}$ and the bit values $0$ and $1$ (cf.~\cite{LehreW12,Doerr18BBC} for a discussion). Unary unbiased variation operators are also referred to as \emph{mutation operators}. 

The characterization of unary unbiased variation operators in Lemma~\ref{lem:characterizationunary} below states that each such operator is uniquely defined via a probability distribution $r_{p}$ over the set $[0..n]:=\{0\} \cup [n]$ describing how many bits (chosen uniformly at random without replacement) are flipped in the argument~$x$ to create a new search point~$y$. Finding a best possible mutation-based algorithm is thus identical to identifying an optimal strategy to select the distribution $r_{p}$. This characterization can also be derived from~\cite[Proposition~19]{DoerrKLW13}, although the original formulation of Proposition~19 in~\cite{DoerrKLW13} requires as search space $[n]^{n-1}$. 

\begin{algorithm2e}[t]
 \textbf{Initialization:}\\
\Indp
	$t \gets 0$\;
		Choose $x(t)$ uniformly at random from $S=\{0,1\}^n$\;
 \Indm
 \textbf{Optimization:}\\	
\Indp
\textbf{repeat}\\
\Indp
		$t \gets t+1$\;
		Evaluate $f(x(t-1))$\;
		Based on $\Big(f(x(0)), \ldots, f(x(t-1)) \Big)$ choose a probability distribution $p_s$ on $[0..t-1]$ and a unary unbiased operator $(p(\cdot|x))_{x \in \{0,1\}^n}$\;
		Randomly choose an index $i$ according to $p_s$\;
		Generate $x(t)$ according to $p(\cdot|x(i))$\;
		\Indm
		\textbf{until} termination condition met\;
		\Indm
 \caption{Blueprint of a unary unbiased black-box algorithm}
\label{alg:unbiased}
\end{algorithm2e}

\begin{lemma}
\label{lem:characterizationunary}\label{LEM:CHARACTERIZATION}
For every unary unbiased variation operator $(p(\cdot|x))_{x \in \{0,1\}^n}$ there exists a probability distribution $r_{p}$ on $[0..n]$ such that for all $x,y\in\{0,1\}^n$ the probability $p(y|x)$ that $(p(\cdot|x))_{x \in \{0,1\}^n}$ samples $y$ from $x$ equals the probability of sampling a random number~$r$ from~$r_{p}$ and then flipping $r$ bits in $x$ to create $y$. On the other hand, each distribution $r_{p}$ on $[0..n]$ induces a unary unbiased variation operator.
\end{lemma}

To prove Lemma~\ref{LEM:CHARACTERIZATION}, we introduce the following notation. 
\begin{definition}\label{flip_r}
Let $r \in [0..n]$. For every $x \in \{0,1\}^n$ the operator $\flip_r$ creates an offspring $y$ from $x$ by selecting $r$ positions $i_1, \ldots, i_r$ in $[n]$ uniformly at random (without replacement), setting $y_i:=1-x_i$ for $i \in \{i_1,\ldots,i_r\}$, and copying $y_i:=x_i$ for all other bit positions $i \in [n]\setminus \{i_1,\ldots,i_r\}$.
\end{definition}

With this notation, Lemma~\ref{LEM:CHARACTERIZATION} states that every unary unbiased variation operator $(p(\cdot|x))_{x \in \{0,1\}^n}$ can be described by Algorithm~\ref{alg:operator}, for a suitably chosen probability distribution $r_{p}$. 

Since it will be needed several times in the remainder of this work, we briefly recall the following simple fact about the expected $\OM$-value of an offspring generated by $\flip_r$.
\begin{remark}\label{rem:expectedflips}
Let $x \in \{0,1\}^n$ and $r \in [0..n]$. The number of $1$-bits that are flipped by the variation operator $\flip_r$ follows a hypergeometric distribution with expectation equal to $r\OM(x)/n$. The expected number of $0$-bits that are flipped by the operator $\flip_r$ equals $r(n-\OM(x))/n$. The expected $\OM$-value of an offspring generated from $x$ by applying $\flip_r$ is thus equal to 
$\OM(x)-r\OM(x)/n+r(n-\OM(x))/n = (1-2r/n)\OM(x)+r$.
\end{remark}

\begin{algorithm2e}[t]
		Choose an integer $r \in [0..n]$ according to $r_{p}$\;
		Sample $y \assign \flip_r(x)$\;
	\caption{The unary unbiased operator $r_{p}$ samples for a given $x$ an offspring by sampling from $r_{p}$ the number $r$ of bits to flip and then applying $\flip_r$ to $x$.}
	\label{algo}
\label{alg:operator}
\end{algorithm2e}
 
To prove Lemma~\ref{LEM:CHARACTERIZATION} we show the following.
\begin{lemma}\label{lem:irgendwas22}
For every unary unbiased variation operator $(p(\cdot|x))_{x \in \{0,1\}^n}$ there exists a exists a probability distribution $r_{p}$ on $[0..n]$ such that for all $x,y\in\{0,1\}^n$
\begin{equation}\label{eq:lemma3}
	p\left(y|x\right)= r_{p}(H(x,y))/\binom{n}{H(x,y)}, 
\end{equation}
where here and henceforth $H(x,y)$ denotes the Hamming distance of $x$ and $y$. 
\end{lemma}

\begin{proof}
	Let $p$ be a unary unbiased variation operator. We first show that for all $x,y_1,y_2\in\{0,1\}^n$ with $H(x,y_1)=H(x,y_2)$, the equality $p(y_1|x)=p(y_2|x)$ holds. This shows that for any fixed Hamming distance $d\in[n]$ and every string $x$, the probability distribution on the $d$-neighborhood $\mathcal{N}_{d}(x):=\big\{ y\in\{0,1\}^n,H(x,y)=d \big\}$ of $x$ is uniform. 
	%

	Using the fact that the bit-wise XOR operator $\oplus$ preserves the Hamming distance, we obtain that for any $y_1,y_2\in \mathcal{N}_{d}(x)$ it holds that 
	\begin{equation*}
		H(x\oplus x,y_1\oplus x)=H(x\oplus x,y_2\oplus x)=d.
	\end{equation*}
	Since $x\oplus x=(0,\ldots,0)$, we thus observe that
	\begin{equation*}
	\sum_{i=1}^{n}(y_1\oplus x)_i=\sum_{i=1}^{n}(y_2\oplus x)_i=d.
	\end{equation*}
	This implies that there exists a permutation $\sigma\in S_n$ such that
	\begin{equation*}
	\sigma(y_1\oplus x)=y_2\oplus x.
	\end{equation*}
	According to the definition of unary unbiased variation operators, $p$ is invariant under "$\oplus$" and "$\sigma$", yielding
	\begin{eqnarray*}
	p(y_1|x)&=& p(y_1\oplus x | x\oplus x)\\
	&=& p(\sigma(y_1\oplus x) |\sigma(x\oplus x))\\
	&=& p(y_2\oplus x | x\oplus x)= p(y_2|x).
	\end{eqnarray*}
	This shows that $p(\cdot|x)$ is uniformly distributed on $\mathcal{N}_d(x)$ for any $d\in[n]$ and any $x \in \{0,1\}^n$. 
	
	For every $x \in \{0,1\}^n$ and every $d\in[n]$ let $p_{(d,x)}$ denote the probability of sampling a specific point at distance $d$ from $x$. That is, for $y_1 \in \mathcal{N}_d(x)$ we have $p(y_1|x)=p_{(d,x)}$. For $x'\neq x$ let $y'$ denote $y_1\oplus (x \oplus x')$. Then by the unbiasedness of $p$ we obtain that 
	\begin{equation*}
		H(y',x')=H(y_1,x)=d, \quad \text{ and }  p_{(d,x')}=p(y'|x')=p(y'\oplus x \oplus x' \mid x'\oplus x \oplus x')=p(y_1|x)=p_{(d,x)}.
	\end{equation*}
	Thus, for all $x,x' \in \{0,1\}^n$ it holds that $p_{(d,x)}=p_{(d,x')}=:p_{(d)}$.
	
	For the unary unbiased variation operator $p$ we can therefore define a distribution $r_p$ on $[0..n]$ by setting 
	\begin{align*}
	r_p(d)
	= \binom{n}{d} p_{(d)}.
	\end{align*}
	For all $x,y\in \{0,1\}^n$ we obtain
	\begin{align*}
	p(y|x)=p_{(d,x)}=p_{(d)}=r_p(d)/\binom{n}{d},
	\end{align*}
	where we abbreviate $d:=H(x,y)$. This shows the desired equation~\eqref{eq:lemma3}.
  \end{proof}

\subsection{Drift Analysis}
\label{SEC:DRIFT}\label{sec:drift}

The main tool in our work is drift analysis, a well-established method in the theory of randomized search heuristics. Drift analysis tries to translate information about the expected progress of an algorithm into information about the expected runtime, that is, the expected hitting time of an optimal solution. We note that typically drift theorems are phrased in a way that they estimate the time a stochastic process on a subset of the real numbers takes to hit zero. This is convenient when regarding as process some distance of the current-best solution of an evolutionary algorithm to the optimum. Of course, via elementary transformations all drift results can be rephrased to hitting times for other targets.

Drift analysis was introduced to the field in the seminal work of He and Yao~\cite{HeY04}. They proved the following additive drift theorem, which assumes a uniform bound on the expected progress.

\begin{theorem}[Additive drift theorem \cite{HeY04}]
\label{THM:ADDITIVEDRIFT}
	Let $(X_t)_{t\ge0}$ be a sequence of non-negative random variables over a finite state space in $\mathbb{R}$. Let $T$ be the random variable that denotes the earliest
	point in time $t \ge 0$ such that $X_t = 0$. If there exist $c,d > 0$ such that
	\begin{equation*}c\le \E\left(X_t-X_{t+1}\mid T>t\right)\le d,\end{equation*}
	then
	\begin{equation*}
	\frac{X_0}{d}\le \E\left(T \mid X_0\right)\le \frac{X_0}{c}.
	\end{equation*}
\end{theorem}

Since the progress of many algorithms slows down the closer they get to the optimum, the requirement of Theorem~\ref{THM:ADDITIVEDRIFT} to establish an \emph{additive} bound on the drift is often not very convenient. For this reason, a \emph{multiplicative} drift theorem that only requires that the expected progress is proportional to the distance was proposed in~\cite{DoerrJW12}. The multiplicative drift theorem is often a good tool for the analysis of randomized search heuristic on the \onemax test function and related problems. However, for the very precise bound that we aim at it in this present work, it does not suffice to approximate the true drift behavior by a linear progress estimate. For this reason, we resort to the most general technique known in the area of drift analysis, which is called \emph{variable} drift, and which assumes no particular behavior of the progress. Variable drift was independently developed in~\cite{MitavskyRC09} and~\cite{Johannsen10}. 

\begin{theorem}[Johannsen's Theorem~\cite{Johannsen10}] \label{U}
	Let $(X_t)_{t\ge0}$ be a sequence of non-negative random variables over a finite state space in $\mathcal{S}\subset\mathbb{R}$ and let $x_{\min}=\min\{x\in\mathcal{S}:x>0\}$. Furthermore, let $T$ be the random variable that denotes the earliest
	point in time $t \ge 0$ such that $X_t = 0$. Suppose that there exists a continuous and monotonically increasing function $h:\mathbb{R}_0^+\to\mathbb{R}_0^+$ such that for all $t<T$ it holds that
	\begin{equation*}
	 \E\left(X_t-X_{t+1}\mid X_t\right)\ge h(X_t).
	\end{equation*}
Then
	\begin{equation*}
	\E\left(T\mid X_0\right)\le \frac{x_{\min}}{h(x_{\min})}+\int_{x_{\min}}^{X_0}\frac{1}{h(x)} dx.
	\end{equation*}
\end{theorem}

We recall that a function $h:\R \rightarrow \R$ is monotonically increasing (decreasing), if $h(r)\leq h(s)$ ($h(r)\geq h(s)$, respectively) holds for all $r,s \in \R$ with $r <s$. 

In Theorem~\ref{U} the assumption that $h$ is continuous is not necessary. This was shown in~\cite{RoweS14} by redoing the original proof, but replacing the clumsy mean-value argument used to show equation~(4.2.1) in~\cite{Johannsen10} by a natural elementary estimate. This result was further used to reprove a not so easily accessible variable drift theorem for discrete search spaces~\cite{MitavskyRC09}. We give below an alternative proof for this discrete drift theorem, which uses the elementary idea to approximate a step function by continuous functions.  

\begin{theorem}[Discrete Variable Drift, upper bound] \label{dis_U}
	Let $(X_t)_{t\ge0}$ be a sequence of random variables in $[0..n]$ and let $T$ be the random variable that denotes the earliest point in time $t \ge 0$ such that $X_t = 0$. Suppose that there exists a monotonically increasing function $h:\{1,\dots,n\}\to\mathbb{R}_0^+$ such that
	\begin{equation*}
	\E\left(X_t-X_{t+1}\mid X_t\right)\ge h(X_t)
	\end{equation*}
	holds for all $t<T$. Then
	\begin{equation*}
	\E\left(T\mid X_0\right)\le \sum_{i=1}^{X_0}\frac{1}{h(i)}.
	\end{equation*}
\end{theorem}

\begin{proof}
For all $0<\eps<1$, define a function $h_{\eps}:[1,n]\to\mathbb{R}^+$ by

	\begin{equation}\label{h_cont}
	h_{\eps}(x):=\left\{
	\begin{array}{l@{\:~ \text{ for}~}l@{\:}}
		h(\lfloor x \rfloor)+\frac{x-\lfloor x \rfloor}{\eps}
		\big( h(\lceil x \rceil)-h(\lfloor x \rfloor) \big)	 & 0<x-\lfloor x \rfloor\le\epsilon,\\
	h(\lceil x \rceil) & \epsilon<x-\lfloor x \rfloor,\\
	\end{array}
	\right.
	\end{equation}
	
	The continuous monotone function $h_{\eps}$ satisfies $	\E(X_t-X_{t+1}\mid X_t)\ge h_{\eps}(X_t)$ for all $t<T$. We compute
	\begin{eqnarray*}
		\int_{1}^{X_0}\frac{1}{h_{\eps}(x)} dx &\le&
		\sum_{i=2}^{X_0} \frac{\eps}{h(i-1)}+\frac{1-\eps}{h(i)}\\
		&=&\frac{\eps}{h(1)}-\frac{\eps}{h(X_0)}+\sum_{i=2}^{X_0} \frac{1}{h(i)}.
	\end{eqnarray*}
	Hence by  Theorem \ref{U}, we have $\E(T\mid X_0) \le \frac{\eps}{h(1)}-\frac{\eps}{h(X_0)} +\sum_{i=1}^{X_0}\frac{1}{h(i)}$ for all $\eps > 0$, which proves the claim. 
  \end{proof}

To prove lower bounds on runtimes, the following variable drift theorem having a similar structure as Theorem~\ref{U} was given in~\cite{DoerrFW11}. We state the result in the original formulation of~\cite{DoerrFW11}, but note that as for Theorem~\ref{U}, the assumptions that $h$ and $c$ are continuous are not necessary. 
As a main difference to Theorem~\ref{U}, we now have the additional requirement that the process with probability one in each round does not move too far towards the target (first condition of the theorem below). 

\begin{theorem}[Variable Drift, lower bound~\cite{DoerrFW11}] \label{L}
	Let $(X_t)_{t\ge0}$ be a decreasing sequence of non-negative random variables over a finite state space in $\mathcal{S}\subset\mathbb{R}$ and let $x_{\min}=\min\{x\in\mathcal{S}:x>0\}$. Furthermore, let $T$ be the random variable that denotes the earliest
	point in time $t \ge 0$ such that $X_t = 0$. Suppose that there exists two continuous and monotonically increasing function $c,h:\mathbb{R}_0^+\to\mathbb{R}_0^+$ such that
	\begin{enumerate} 
		\item $X_{t+1}\ge c(X_t)$,
		\item $\E(X_t-X_{t+1}\mid X_t)\le h(c(X_t))$.
	\end{enumerate}
	Then 
	\begin{equation*}
	\E\left(T\mid X_0\right)\ge \frac{x_{\min}}{h(x_{\min})}+\int_{x_{\min}}^{X_0}\frac{1}{h(x)} dx.
	\end{equation*}
\end{theorem}

The first condition of the theorem above, that with probability one no progress beyond a given limit is made, is a substantial restriction to the applicability of the theorem, in particular, when working with evolutionary algorithms, where often any parent can give birth to any offspring (though usually with very small probability). In~\cite{DoerrFW11}, this problem was overcome with the simple argument that a progress of more than $\sqrt x$ from a search point with fitness distance $x$ occurs with such a small probability that it does with high probability not occur in a run of typical length. 

For the precise bounds that we aim at, such a simple argument cannot work. We therefore show a  version of the lower bound variable drift theorem which allows larger jumps provided that they occur sufficiently rarely. To prove this result, we cannot use a blunt union bound over all bad events of too large jumps, but have to take these large jumps into account when computing the additive drift in a proof analogous to the one of~\cite{DoerrFW11}. This idea was already used in~\cite{GiessenW16}, where a similar, but more technical drift theorem was derived. The result of~\cite{GiessenW16} is valid for arbitrary domains, whereas we restrict ourselves to processes over the non-negative integers, but this is not the main reason for the simplicity of our result.

\begin{theorem}[Discrete Variable Drift, lower bound] \label{dis_L}
	Let $(X_t)_{t\ge0}$ be a sequence of decreasing random variables in $[0..n]$ and let $T$ be the random variable that denotes the earliest point in time $t \ge 0$ such that $X_t = 0$.
	Suppose that there exists two monotonically increasing functions $c:[n]\to [0..n]$ and $h:[0..n]\to\mathbb{R}_0^+$, and a constant $0\le p<1$ such that
	\begin{enumerate} 
		\item $X_{t+1}\ge c(X_t)$ with probability at least $1-p$ for all $t<T$,
		\item $\E(X_t-X_{t+1}\mid X_t)\le h(c(X_t))$ holds for all $t<T$.
	\end{enumerate}
	Let $g:[0..n] \to \mathbb{R}_0^+$ be the function defined  by $g(x)=\sum_{i=0}^{x-1}\frac{1}{h(i)}$. Then 
	\begin{equation*}
	\E(T\mid X_0)\ge g(X_0)-\frac{g^2(X_0)p}{1+g(X_0)p}.
	\end{equation*}
\end{theorem}
\begin{proof}
	The function $g$ is strictly monotonically increasing. We have $g(X_t)=0$ if and only if $X_t=0$. Using condition $1$ and the monotonicity of $h$, in the case $X_{t+1}\ge c(X_t)$ we have
	\begin{equation*}
	g(X_t)-g(X_{t+1})=\sum_{i=X_{t+1}}^{X_t-1}\frac{1}{h(i)}\le \frac{X_t-X_{t+1}}{h(X_{t+1})} \le \frac{X_t-X_{t+1}}{h(c(X_t))},
	\end{equation*}
	and otherwise
	\begin{equation*}
	g(X_t)-g(X_{t+1})\le g(X_t)\le g(X_0).
	\end{equation*}
	Using inequality $\E(X_t-X_{t+1}\mid X_t)\le \E(X_t-X_{t+1}\mid X_{t+1}\ge c(X_t))$ and condition 2, we have
	\begin{eqnarray*}
	\E\left(g(X_t)-g(X_{t+1})\mid g(X_t)\right)&\le&
	\E\left(\frac{X_t-X_{t+1}}{h(c(X_t))}\mid X_{t+1}\ge c(X_t)\right)(1-p)+g(X_0)p\\
	&\le& 1+g(X_0)p.
	\end{eqnarray*}
	Applying the additive drift theorem~\ref{THM:ADDITIVEDRIFT} to $g(X_t)_{t\ge 0}$, we obtain
	\begin{equation*}
	\E\left(T\mid X_0\right)=\E\left(T\mid g(X_0)\right)\ge \frac{g(X_0)}{1+g(X_0)p}=g(X_0)-\frac{g^2(X_0)p}{1+g(X_0)p}.
	\end{equation*}
  \end{proof}

To apply the drift theorem above (or Theorem~\ref{L}) one needs to guess a suitable function~$h$ such that $h \circ c$ is an upper bound for the drift. The following simple reformulation overcomes this difficulty by making $h(x)$ simply an upper bound for the drift from a state~$x$. This also makes the result easier to interpret. The influence of large jumps, as quantified by $c$, now is that the runtime bound is not anymore the sum of all $h(x)^{-1}$ as in the upper bound theorem, but of all $h(\mu(x))^{-1}$, where $\mu(x)$ is the largest point $y$ such that $c(y) \le x$. 
So in simple words, we have to replace the drift at $x$ pessimistically by the largest drift among the points from which we can go to $x$ (or further) in one round with probability more than $p$.

\begin{theorem}[Discrete Variable Drift, lower bound] \label{dis_L'}
	Let $(X_t)_{t\ge0}$ be a sequence of decreasing random variables in  $[0..n]$ and let $T$ be the random variable that denotes the earliest point in time $t \ge 0$ such that $X_t = 0$.
	Suppose that there exists two functions $c:\{1,\dots,n\}\to [0..n]$ and monotonically 
	increasing  $h:[0..n] \to\mathbb{R}_0^+$, and a constant $0\le p<1$ such that 
	\begin{enumerate} 
		\item $X_{t+1}\ge c(X_t)$ with probability at least $1-p$ for all $t<T$,
		\item $\E(X_t-X_{t+1}\mid X_t)\le h(X_t)$ holds for all $t<T$.
	\end{enumerate}
	Let $\mu:[0..n] \to [0..n]$ be the function defined by
	$\mu(x)=\max\{i|c(i)\le x\}$ and $g:[0..n]\to \mathbb{R}_0^+$ be the function defined  by $g(x)=\sum_{i=0}^{x-1}\frac{1}{h(\mu(i))}$. Then 
	\begin{equation*}
	\E(T\mid X_0)\ge g(X_0)-\frac{g^2(X_0) p}{1+g(X_0)p}.
	\end{equation*}
\end{theorem}

\begin{proof}
	Let $\hat{c}:[0..n]\to [0..n], r \mapsto \hat{c}(r):=\min\{i|\mu(i)\ge r\}$.
	By definition, $\hat{c}$ is monotonically increasing.
	Let $r \in [n]$. Since $r\in\{i|c(i)\le c(r)\}$, we have $\mu(c(r))\ge r$, which implies $c(r)\in\{i|\mu(i)\ge r\}$. Hence $\hat{c}(r)\le c(r)$. This shows that we have $X_{t+1}\ge \hat{c}(X_t)$ for all $t < T$.
	
	The definition of $\hat{c}$ implies that $\mu(\hat{c}(r))\ge r$. Together with the monotonicity of $h$, we obtain $\E(X_t-X_{t+1}\mid X_t)\le h(X_t)\le h(\mu(\hat{c}(X_t)))$. 
	Let $\hat{h}:[0..n]\to \mathbb{R}_0^+, r \mapsto \hat{h}(r):=h(\mu(r))$. It is easy to see from the definition that $\mu$ is monotonically increasing, therefore, $\hat{h}$ is also monotonically increasing and it satisfies $\E(X_t-X_{t+1}\mid X_t)\le \hat{h}(\hat{c}(X_t))$. Applying Theorem \ref{dis_L} to $\hat h$ and $\hat c$ shows the claim.
	  \end{proof}

\section{Maximizing Drift is Near-Optimal}
\label{sec:maxDrift}\label{SEC:MAXDRIFT}

The goal of this section is to show that the algorithm which maximizes the expected progress over the best-so-far search point is optimal, apart from lower-order terms $o(n)$, among all unary unbiased black-box algorithms. Consequently, its expected optimization time is essentially the unary unbiased black-box complexity. 

\subsection{The Drift Maximizing Algorithm}
\label{sec:ouralgo}

We regard as drift maximizing algorithm the algorithm summarized in Algorithm~\ref{alg:algo}. We denote this algorithm by $A^*$. $A^*$ starts by querying a uniform solution $x \in \{0,1\}$. 
In each iteration of the main loop the algorithm generates a new solution $y$ from $x$ by flipping exactly $R(\OM(x))$ bits in $x$, where $R:[0..n] \to [0..n]$ is a function that assigns to each fitness value the number of bits that should be flipped in a search point of this quality. We choose $R$ such that the expected progress (\emph{drift}) is maximized. When there is more than one value maximizing the drift, $R$ chooses the smallest among these. That is, 
\begin{equation}
	R(\OM(x)):= \min \left\{ \arg\max \E_{y \assign \flip_r(x)}\big(\max\{ \OM(y)-\OM(x),0 \}\big) \mid r \in [0..n] \right\}.\label{def_R}
\end{equation}

\begin{algorithm2e}[t]
	Choose $x$ uniformly at random from $S=\{0,1\}^n$\;
	\For{$t=1,2,\ldots$}{
		\Indp
		$y \assign \flip_{R(\OM(x))}(x)$\;
		\lIf{$\OM(y)\geq \OM(x)$\label{line:select}}{$x \assign y$\;}
		\Indm
	}
	\caption{Structure of our algorithm}
	\label{alg:algo}\label{ALG:ALGO}
\end{algorithm2e}
In this definition, we make use of the fact that the expected progress $\E_{y \assign \flip_r(x)}\big(\max\{ \OM(y)-\OM(x),0 \}\big)$ depends only on the fitness of $x$ but not on its structure. This is due to the symmetry of the function $\OM$.
The offspring $y$ replaces its parent $x$ if and only if $\OM(y) \ge \OM(x)$.

Note here that the function $R$ is deterministic, i.e., we only make use of a unary unbiased mutation operator which deterministically depends on the fitness of the current-best solution. Note further that the search point kept in the memory of algorithm $A^*$ is always a best-so-far solution. $A^*$ can be seen as an RLS-variant with fitness-dependent mutation strength. 

\subsection{Main Result and Proof Strategy}\label{sec:driftdefs}
The main result of this entire section is the following statement, which says that the expected runtime of $A^*$ cannot be much worse than the unary unbiased black-box complexity of \onemax.  

\begin{theorem}\label{thm:main22}
Let $A$ be a unary unbiased black-box algorithm. Denote by $T(A)$ its runtime on \onemax and by $T(A^*)$ the runtime of $A^*$. Then
$\E(T(A)) \ge \E(T(A^*)) - O(n^{2/3}\ln^9(n))$.
\end{theorem}
For the proof of Theorem~\ref{thm:main22} we derive in Section~\ref{sec:lowerany} a lower bound for the expected runtime of any unary unbiased algorithm, cf. Theorem~\ref{thm:LBany}. We then prove in Section~\ref{sec:driftmaxupper} an upper bound for the expected runtime of $A^*$, cf. Theorem~\ref{thm:UBdriftmax}. For both statements, we use the variable drift theorems presented in Section~\ref{sec:drift}. We therefore need to define suitable drift functions which bound from below the expected progress that can be made by algorithm $A^*$ and from above the maximal expected progress that any unary unbiased algorithm can make at every step of the optimization process. This is the purpose of the remainder of this subsection. 

\subsubsection{The Distance Function}
Before we define the drift functions, we first note that in order to maximize the function \OM, an optimal algorithm may choose to first minimize the function, and to then flip all bits at once to obtain the optimal $\OM$-solution. Instead of regarding the \emph{maximization of the function $\OM$,} we therefore regard in the following the problem of \emph{minimizing the distance function}~$d$, which assigns to each string $x$ the value $d(x):=\min\{n-\OM(x),\OM(x)\}$. The black-box complexities of both problems are almost identical, as the following lemma shows. 

\begin{lemma}\label{lem:BBCdistance}
The unary unbiased black-box complexities of maximizing \onemax is at least as large as that of minimizing $d$ and it is larger by at most one.
\end{lemma} 

\begin{proof}
For the first statement, it suffices to observe that we can simulate the optimization of $d$ when \onemax-values are available. The second statement follows from the already mentioned fact that once we have found a string $x$ of distance value $d(x)=0$, then either $x$ or its bit-wise complement $\bar{x}$ has maximal \onemax-value.
 \end{proof}

\subsubsection{Drift Expressions}

For the definition of the drift functions used in the proofs of Theorems~\ref{thm:LBany} and~\ref{thm:UBdriftmax}, we use the following notation. For a unary unbiased algorithm $A$ we denote by $(x(0),x(1),\ldots,x(t))$ the sequence of the first $t+1$ search points evaluated by $A$. For every such sequence of search points, we abbreviate by 
$$X_t:=\min\{d(x(i))|i \in [0..t]\}$$ 
the quality of a best-so-far solution with respect to the distance function $d$. Note that for all $t\geq 0$ it holds that $X_t \geq X_{t+1}$, i.e., the sequence $(X_t)_{t\geq 0}$ is monotonically decreasing in $t$. For each $k \in [0..n]$ let $\mathcal{H}(t,k)$ be the collection of all sequences $(x(0),x(1),\ldots,x(t))$ of search points for which $X_t(x(0),x(1),\ldots,x(t))=k$. Abusing notation, we write $x \in ((x(0),x(1),\ldots,x(t)))$ when $x=x(i)$ for some $i \in [0..t]$.

Denoting by $\mathcal{U}$ the set of all unary unbiased operators acting on $\{0,1\}^n$, the maximal possible drift that can be achieved by a unary unbiased variation operator when the best-so-far distance is equal to $k \in [0..n]$ is equal to  
	\begin{eqnarray*}
	\hat{h}(k)&:=&\max\left\{\max\left\{k-\E(\min\{k,d(\tau(x))\})\mid \tau \in \mathcal{U},x\in H\right\}\mid t \in \N, H\in\mathcal{H}(t,k)\right\}.
	\end{eqnarray*} 
	Using Lemma~\ref{lem:characterizationunary} it is not difficult to show that
\begin{eqnarray}
		\hat{h}(k)&=&\max\left\{k-\E(\min\{k,d(\flip_{r}(x))\})\mid r\in[0..n], x \in \{0,1\}^n \text{ with } d(x) \ge k \right\}. \label{def:h_hat}
	\end{eqnarray}
	
To obtain a monotonically increasing function (as required by Theorem~\ref{dis_L'}), we set
 	\begin{eqnarray}	
	h(k):=\max\{\hat{h}(i) \mid i \in[0..k] \}. \label{def:h}
	\end{eqnarray}	
Note that $h(k) \ge \hat{h}(k)$ for all $k \in [0..n]$.	

Finally, we set 
\begin{eqnarray}
	\tilde{h}(k)&:=&\max\{\E(\max\{\OM(\flip_r(x))-\OM(x),0\})\mid x \in \{0,1\}^n \text{ with } \OM(x)=n-k, r\in[0..n]\},\label{def:h_tilde}
	\end{eqnarray} 
the maximal $\OM$-drift that can be obtained from a search point whose \OM-value is exactly equal to $n-k$. We certainly have $\tilde{h}(k) \le h(k)$ for all $k\le n/2$. However, we will show in Section~\ref{sec:bestsofar} that for all values $k \le n/2-n^{0.6}$ the difference between the functions $h$ and $\tilde{h}$ is small, showing that we can approximate the maximal drift $h$ by mutating a best-so-far solution. This will be the key step in proving the upper bound for the expected runtime of algorithm $A^*$.  We also notice that $\tilde{h}(k)\ge \max\{\OM(\bar{x})-\OM(x),0\}\ge 2k-n$. Therefore for any $x\in\{0,1\}^n$ with $\OM(x)=n-k$ we have $\OM(x)+\tilde{h}(k)\ge \max\{\OM(x),\OM(\bar{x})\}\ge n/2$, so that $A^*$ very quickly has a search point of function value $\OM(x)\ge n/2$. Informally, the interesting part of the runtime analysis for $A^*$ is therefore the fitness increase from a value around $n/2$ to $n$.

\subsection{A Lower Bound for all Unary Unbiased Algorithms}
\label{sec:lowerany}

Before proving the lower bound, we first introduce the following lemma arguing that the probability to make a large fitness gain is bounded by a small probability, as required by the condition to apply Theorem~\ref{dis_L'}. 

\begin{lemma}\label{prob-large-jump}
	 There exists an $n_0 \in \N$ such that, for all $n \ge n_0$, for all $r\in [0..n]$, and for all $x\in\{0,1\}^n$, it holds that 
	\begin{equation}\label{tilde_c}
	\Pr\left(d\left(\flip_r(x)\right)\ge \tilde{c}\left(d(x)\right)\right)\ge 1-n^{-4/3}\ln^7(n),
	\end{equation}
	where 
		\begin{equation*}
		\tilde{c}:[n] \to [0..n], i \mapsto \tilde{c}(i):=\left\{
		\begin{array}{l@{\:~ \text{ for}~}l@{\:}}
		i-\sqrt{n}\ln n & i\ge  n/6,\\
		i-\ln^2(n) & n^{1/3}\le i < n/6,\\
		i-1 & i< n^{1/3}.\\
		\end{array}
		\right.
		\end{equation*}
\end{lemma}

\begin{proof}
Set $p:=n^{-4/3}\ln^7(n)$. To show~\eqref{tilde_c}, we first note that we can assume without loss of generality that $\OM(x)\ge n/2$. This is due to the symmetry of the distance function. In addition, we can assume that $0<r\le n/2$, because $d(\flip_r(x))$ and $d(\flip_{n-r}(x))$ are identically distributed. 

We make a case distinction according to the size of $d:=d(x)$. For all different cases we note that the event $d(\flip_r(x))<d$ happens in two cases, namely if $\OM(\flip_r(x))>\OM(x)$ or if $\OM(\flip_r(x))<n-\OM(x)$. We denote by $Z$ the number of good flips in an application of $\flip_r$ to $x$; i.e., the number of bits flipping from $0$ to~$1$. $Z$ follows a hypergeometric distribution and $\E(Z)=rd(x)/n$, as discussed in Remark~\ref{rem:expectedflips}. 

\underline{\textbf{Case 1: $d(x)\ge n/6$.}} We first regard the case that $d := d(x)\ge n/6$. The Chernoff bound (e.g., the variants presented in Theorems~1.11 and~1.17 in~\cite{Doerr11bookchapter}) applied to $Z$ show that, for all $\lambda>0$,  
$\Pr(Z>\E(Z)+\lambda)\le \exp(-2\lambda^2/n)$ and 
$\Pr(Z<\E(Z) -\lambda)\le \exp(-2\lambda^2/n)$.
Using that $\OM(\flip_r(x))=\OM(x)+2Z-r$ and that $n-\OM(x)\le \E(\OM(\flip_r(x)))\le \OM(x)$, we obtain that 
\begin{align}
			\Pr\left( \OM(\flip_r(x))<n-\OM(x)-2\lambda\right) \label{flip<E-lambda}
& \le \Pr\left( \OM(\flip_r(x))<\E(\OM(\flip_r(x)))-2\lambda\right)\\
& \nonumber =   \Pr\left( \OM(x)+2Z-r<\OM(x)+E(2Z)-r-2\lambda\right) \\
& \nonumber =		\Pr\left( 2Z<E(2Z)-2\lambda\right)
\le \exp(-2\lambda^2/n),
\end{align}
and that
\begin{align}
\label{flip>E+lambda}
					\Pr\left( \OM(\flip_r(x))>\OM(x)+2\lambda\right)
		& \le	\Pr\left( \OM(\flip_r(x))>\E(\OM(\flip_r(x)))+2\lambda\right)\\
		& \nonumber =		\Pr\left(2Z>E(2Z)+2\lambda\right)
		\le \exp(-2\lambda^2/n).
\end{align}
From these two inequalities we conclude that 
\begin{align*}
\Pr(d(\flip_r(x))<\tilde{c}(d(x))
& =\Pr(d(\flip_r(x))<d-\sqrt{n}\ln n)
\le 2\exp(-\ln^2(n)/2)<p
\end{align*}
for all $d\ge n/6$.

\underline{\textbf{Case 2: $n^{1/3} \le d(x) < n/6$.}} By our assumptions $r\le n/2$ and $\OM(x)\ge n/2$ we obtain from Remark~\ref{rem:expectedflips} that $\E(\OM(\flip_r(x))) = (1-\tfrac{2r}{n})\OM(x) + r \ge (1-\tfrac{2r}{n})\tfrac{n}{2} + r = n/2$. Using Chernoff bounds (e.g., Theorems~1.11 and~1.17 in~\cite{Doerr11bookchapter}) we thus get  
$\Pr(\OM(\flip_r(x))<n/6)\le \exp(-2(n/2-n/6)^2/n)=o(p)$. 

Consider the event $\OM(\flip_r(x))\ge \OM(x)+\ln^2(n)$. It intrinsically requires that $r\ge \ln^2(n)$. We apply the Chernoff bound from Corollary 1.10~(b) in \cite{Doerr11bookchapter} to $Z$ and obtain that
\begin{eqnarray}
\Pr\left(\OM(\flip_r(x))>\OM(x)\right)&=&\Pr\left(Z> \frac r2 \right)
=\Pr\left(Z>\frac{n}{2d}\E(Z)\right)\le\left(\frac{2de}{n}\right)^{r/2}
\label{chernoff_Z}\\
&\le&\left(\frac{e}{3}\right)^{r/2}=O(n^{-\Omega(\ln(n))})=o(p).\nonumber
\end{eqnarray}
Therefore 
\begin{align*}
\Pr(d(\flip_r(x))<\tilde{c}(d(x))
& =\Pr(d(\flip_r(x))<d-\ln^2 n)\\
& \le \Pr(\OM(\flip_r(x))<n/6)+ \Pr\left(\OM(\flip_r(x))>\OM(x)\right)
=o(p)
\end{align*} 
for all $n^{1/3}\le d<n/6$.

\underline{\textbf{Case 3: $d(x) < n^{1/3}$.}} We finally consider the case $d<n^{1/3}$. Applying this condition to the first line of Equation~\eqref{chernoff_Z}, we obtain $\Pr(\OM(\flip_r(x))>\OM(x))\le \Theta(n^{-4/3}) = o(p)$ for all $r\ge 4$. 
For $r<4$ we consider the operators separately and observe that 
\[
\Pr\left(\OM(\flip_3(x))>\OM(x)\right)= \frac{\binom{d}{2}\binom{n-d}{1}+\binom{d}{3}}{\binom{n}{3}}=O(n^{-4/3}) = o(p)
\]
and
\[
\Pr\left(\OM(\flip_2(x))>\OM(x)\right)= \frac{\binom{d}{2}}{\binom{n}{2}}=O(n^{-4/3}) = o(p).
\]
Thus, altogether, we obtain that 
\begin{align*}
\Pr(d(\flip_r(x))<\tilde{c}(d(x))
& =\Pr(d(\flip_r(x))<d-1)\\
& \le\Pr(\OM(\flip_r(x))<n/6)+ \Pr\left(\OM(\flip_r(x))>\OM(x)\right)=o(p)
\end{align*}
for all $r\ge 2$ and $d<n^{1/3}$. Needless to say that $\Pr(d(\flip_1(x))<\tilde{c}(d(x))=0$.
This proves Equation \eqref{tilde_c}.
\end{proof}

Using Lemma~\ref{prob-large-jump} and Theorem~\ref{dis_L'}, we are now ready to show the following lower bound. 
\begin{theorem}\label{thm:LBany}
Let $s:=n/2-n^{0.6}$. The expected runtime of any unary unbiased algorithm $A$ on the \OneMax problem satisfies 
\begin{equation*}
E\left(T_{A}\right)\ge \sum_{x=1}^{s}\frac{1}{h(x)}-\Theta(n^{2/3}\ln^9(n)).
\end{equation*}	 
\end{theorem} 


\begin{proof}
For convenience, we assume that $n$ is sufficiently large. 
Let $A$ be a unary unbiased algorithm. We recall that by $(x(0),x(1),\dots)$ we denote the sequence of search points evaluated by $A$ and by $X_t=\min\{d(x(i))\mid i \in [0..t]\}$ the best-so-far distance after first $t$ iterations of the main loop. 
Let $T:=\min\{t\in\mathbb{N}\mid X_t=0\}$. Then $T\le T_{A} := \min\{t\in\mathbb{N}\mid\OM(x(t))=n\}$. 
We prove a lower bound on $T$. 

Every unary unbiased black-box algorithm has to create its first search point uniformly at random. This initial search point $x_0$ has expected fitness $\E(\OM(x_0))=n/2$. By Chernoff's bound (we can use, for example, the variant presented in Theorem~1.11 in~\cite{Doerr11bookchapter}) it furthermore holds that $\Pr(X_0<s)=\Pr(|\OM(x_0)-n/2|>n^{0.6})\le 2\exp(-2n^{0.2})$. This probability is small enough such that even optimistically assuming $T_A=0$ whenever $X_0<s$, the contribution of such events affect the lower bound by a term of $O(n^{2/3}\ln^9(n))$. We can therefore safely assume that $X_0\ge s$.

	For all $i \in [n]$ let $c(i):=\min\{\tilde{c}(j) \mid j\ge i\}$, where $\tilde{c}$ is the function defined in Lemma~\ref{prob-large-jump}. For all $r\in[0..n]$, $x\in\{0,1\}^n$, and all distance levels $d'\le d(x)$ it holds that 
	\[c(d')\le \tilde{c}(d(x)) \text{ and }\Pr\left(d(\flip_r(x))>c(d')\right)\ge \Pr\left(d(\flip_r(x))>\tilde c(d(x))\right) \ge 1-p.\] 
	By Lemma \ref{lem:characterizationunary}, this statement can be extended to arbitrary unary unbiased variation operators. Therefore, 
	\[
	\Pr\left(X_{t+1}>c(X_t)\right)\ge 1-p \text{ for all } t \in \N.
	\] 
	We apply Theorem \ref{dis_L'} to $c$ and $h$. We first compute $\mu$ as in Theorem \ref{dis_L'}. By definition,
	$\mu(i)=\max\{x\mid c(x)\le i\}=\max\{x\mid\min\{\tilde{c}(y)\mid y\ge x\}\le i\}=\max\{x\mid \tilde{c}(x)\le i\}$, giving
	\begin{equation*}
	\mu(i):=\left\{
	\begin{array}{l@{\:~ \text{ for}~}l@{\:}}
	i+1 & i< n^{1/3}-\ln^2(n),\\
	i+\ln^2(n) & n^{1/3}-\ln^2(n) \le i < n/6-\sqrt{n}\ln n,\\
	i+\sqrt{n}\ln n & n/6-\sqrt{n}\ln n \le i < n/2-\sqrt{n}\ln n,\\
	\lfloor n/2 \rfloor &  n/2-\sqrt{n}\ln n \le i \le n/2.
	\end{array}
	\right.
	\end{equation*}
	According to Theorem \ref{dis_L'}, we can thus bound $\E(T\mid X_0)$ by
		\[
		E\left(T\mid X_0\right)\ge g(X_0)-\frac{g^2(X_0)p}{1+g(X_0)p} \text{ with }  g(x)=\sum_{i=1}^{x-1}\frac{1}{h(\mu(x))}.
		\]
	Since $h(\mu(x))\ge \tilde{h}(\mu(x))\ge\E(\max\{\OM(\flip_1(x))-\OM(x),0\}\mid x \in \{0,1\}^n \text{ with } \OM(x)=n-\mu(x))=\mu(x)/n \ge x/n$, we obtain $g(X_0)\le\sum_{i=1}^{X_0-1} \frac{n}{i}=O(n\ln(n))$. Therefore $\frac{g^2(X_0)p}{1+g(X_0)p}=O(n^{2/3}\ln^9(n))$ for $p=n^{-4/3}\ln^7(n)$. Using the monotonicity of $h$ and the fact that all summands are positive, we estimate $g(X_0)$ by 
		\begin{eqnarray*}
			\sum_{x=0}^{X_0-1}\frac{1}{h(\mu(x))}
			&\ge&\sum_{x=1}^{n^{1/3}-\ln^2(n)}\frac{1}{h(x)}
			\quad+\sum_{x=n^{1/3}}^{n/6-\sqrt{n}\ln n+\ln^2(n)}\frac{1}{h(x)}
			\quad+\sum_{x=n/6}^{X_0}\frac{1}{h(x)}\\
			&\ge& \sum_{x=1}^{X_0}\frac{1}{h(x)} - \frac{\ln^2(n)}{h(n^{1/3}-\ln^2(n))}- \frac{\sqrt{n}\ln n - \ln^2(n)}{h(n/6-\sqrt{n}\ln n+\ln^2(n))}\\
			&\ge&  \sum_{x=1}^{X_0}\frac{1}{h(x)}-\Theta(n^{2/3}\ln^2(n)).\\
		\end{eqnarray*}
Therefore we obtain $E(T_{A})\ge \E(T\mid X_0\ge s)-O(n^{2/3}\ln^9(n)) \ge \sum_{x=1}^{s}\frac{1}{h(x)}-\Theta(n^{2/3}\ln^9(n)).$

\end{proof}

\subsection{Upper Bound for the Drift Maximizer}\label{sec:driftmaxupper}
The lower bound in Theorem~\ref{thm:LBany} also holds for drift-maximizer $A^*$ described in the beginning of this section. We next show that $A^*$ achieves this runtime bound apart from the lower order term $\Theta(n^{2/3}\ln^9(n))$. 

\begin{theorem}\label{thm:UBdriftmax}
	Let $s:=n/2-n^{0.6}$. The expected runtime of algorithm $A^*$ on \OneMax 
	satisfies 
	\begin{equation*}
	E\left(T_{A^*}\right)\le \sum_{x=1}^{s}\frac{1}{h(x)}+\Theta(n^{0.6}).
	\end{equation*}	
\end{theorem}

From the variable drift theorem, Theorem~\ref{dis_U}, we easily get $\sum_{x=1}^{s}\frac{1}{\tilde{h}(x)}$ as upper bound for the expected runtime of $A^*$. We therefore need to show that the difference between $h(X_t)$ and $\tilde{h}(X_t)$ is small. This is the purpose of the next subsection.

\subsubsection{Maximizing Drift by Mutating a Best-So-Far Solution}
\label{sec:bestsofar}

As mentioned above, we show that the expected drift of any unary unbiased algorithm cannot be significantly better than that of $A^*$. The main result of this subsection is the following lemma. 

\begin{lemma}\label{lem:compare_h}
For sufficiently sufficiently large $n$ and $k\le n/2-n^{0.6}$ it holds that $0\le h(k)-\tilde{h}(k)\le n\exp(-\Omega(n^{0.2}))$.
\end{lemma}
 
We start our proof of Lemma~\ref{lem:compare_h} by observing that the expected $\OM$-value of the search point obtained from mutating and selecting the best of parent and offspring is strictly increasing with the quality of the parent. The proof is by induction. The base case is covered by the following lemma. 

\begin{lemma}\label{lem:domi}\label{LEM:DOMI}
	Let $x,y\in \{0,1\}^n$ with $\OM(y)=\OM(x)+1\ge n/2$ and let $r\in [0..n/2]$. 
	For all $t \ge 1$ it holds that
	\begin{equation}\label{ineq:expoffspring}
	\Pr\left(\OM(\flip_r(x))=\OM(y)+t\right) \le \Pr\left(\OM(\flip_{r-1}(y))=\OM(y)+t\right).
	\end{equation}
\end{lemma}

\begin{proof}
	We first notice that for $t>r-1$ both two probabilities are zero. We can therefore assume that $1\le t\le r-1$.  
	Let $d:=n-\OM(y)$, and let $i$ be the number of zeros in $y$ that $\flip_{r-1}$ flips from zero to one. Then there are $r-1-i$ ones that flip to zero. We thus have $\OM(\flip_{r-1}(y))=\OM(y)+t$ if and only if $i-(r-1-i)=t$; i.e., if and only if $i=(t+r-1)/2$. By the same reasoning $\OM(\flip_r(x))=\OM(y)+t$ if and only if $i'-(r-i')=t+1$ for $i'$ being the number of zeros flipped by $\flip_r$. This implies $i'=(r-1+t)/2=i+1$. We thus obtain 
	\begin{align*}
	 \Pr\left(\OM(\flip_{r-1}(y))=\OM(y)+t\right) 
			 &=\frac{\binom{d}{i}\binom{n-d}{r-1-i}}{\binom{n}{r-1}} \\
	 \Pr\left(\OM(\flip_r(x))=\OM(y)+t\right)
			 &=\frac{\binom{d+1}{i+1}\binom{n-d-1}{r-(i+1)}}{\binom{n}{r}}
	\end{align*}

	To show equation~\eqref{ineq:expoffspring}, we abbreviate $j:=r-1-i$ and use the facts that 
	$\binom{d+1}{i+1}=\binom{d}{i}\frac{d+1}{i+1}$,
	$\binom{n-d}{j}=\binom{n-d-1}{j}\frac{n-d}{n-d-j}$, and
	$\binom{n}{i+j+1}=\binom{n}{i+j}\frac{n-i-j}{i+j+1}$ to obtain that 
	\[
	\frac{\frac{\binom{d+1}{i+1}\binom{n-d-1}{r-(i+1)}}{\binom{n}{r}}}{\frac{\binom{d}{i}\binom{n-d}{r-1-i}}{\binom{n}{r-1}}}
	=\frac{(d+1)(n-d-j)(1+i+j)}{(i+1)(n-d)(n-i-j)}.
	\]
	We aim to show the above ratio less or equal to $1$. To this end, we compute the difference between the numerator and the denominator, and obtain $(d+1)(n-d-j)(1+i+j)-(i+1)(n-d)(n-i-j)=(1+i+j+d-n)(n+in-(1+i+j)d)-j)=(r+d-n)(n+in-rd-j)$. 
The first factor in this expression is negative, since both $r\le n/2$ and $d\le n/2$. The second factor is positive, because $n>j$, $i>r/2$ and $d\le n/2$ implying that $in-rd > (r/2)n - r (n/2) \ge 0$. 
 \end{proof}

We now regard the case that the same number $r$ of bits are flipped in the two strings~$x$ and~$y$.  
\begin{lemma}\label{lem:mono_h}
	Let $x,y\in \{0,1\}^n$ with $\OM(y)=\OM(x)+1$ and let $r\in [0..n]$. 
	It holds that 
	\begin{equation}\label{ineq:expprogress}
	\E\left(\max\{\OM(\flip_{r}(x))-\OM(x),0\}\right)\ge\E\left(\max\{\OM(\flip_{r}(y))-\OM(y),0\}\right).
	\end{equation}
\end{lemma}

\begin{proof}
	Since permutation on bit-positions does not affect the analysis, we assume that $x$ and $y$ are of the following form. 
	\begin{equation}\label{form_of_x_y}
	x =\underbrace{11\cdots 11}_{\frac{n}{2}+\delta}  \underbrace{00\cdots 00}_{\frac{n}{2}-\delta-1} 0 \text{ and }	
	y =\underbrace{11\cdots 11}_{\frac{n}{2}+\delta}  \underbrace{00\cdots 00}_{\frac{n}{2}-\delta-1} 1.
	\end{equation}
	For any $r$-sized subset $S$ of $[n]$ let $x_S$ ($y_S$) denote the offspring of $x$ ($y$) in which the $r$ positions in $S$ are flipped. Then $x_S$ and $y_S$ differ only in the last bit and we have $\OM(x_S)-\OM(y_S)\in\{-1,1\}$ for all $S$. Therefore 
	\begin{equation*}
	\max\{\OM(\flip_{r}(x)),\OM(x)\}-\max\{\OM(\flip_{r}(y)),\OM(y)\}\ge -1 \text{ for all } S
	\end{equation*}
	Using that $\OM(y)-\OM(x)=1$ we obtain
	\begin{eqnarray*}
	&&	\E\left(\max\{\OM(\flip_{r}(x))-\OM(x),0\}\right)-\E\left(\max\{\OM(\flip_{r}(y))-\OM(y),0\}\right)\\
	&=&\E\left(\max\{\OM(\flip_{r}(x)),\OM(x)\}-\OM(x)\right)-\E\left(\max\{\OM(\flip_{r}(y)),\OM(y)\}-\OM(y)\right)\\
	&=&\E\left(\max\{\OM(\flip_{r}(x)),\OM(x)\}-\max\{\OM(\flip_{r}(y)),\OM(y)\}\right)+1\ge 0.
	\end{eqnarray*}
	  \end{proof}

We now extend the last two lemmas to the case $\OM(y)-\OM(x)>1$. 

\begin{corollary}\label{cor:optoperator}
	Let $x,y\in \{0,1\}^n$ with $\OM(y)>\OM(x)$.
	\begin{enumerate}
	\item if $\OM(x)\ge n/2$, $r \in [0..n/2]$, and $r'=\max\{r-(\OM(y)-\OM(x)),0\}$. Then for all $t\in\mathbb{N}_{\ge 1}$  we have
	\begin{align}\label{33eq}
	\Pr\left(\OM(\flip_r(x))=\OM(y)+t\right) \le \Pr\left(\OM(\flip_{r'}(y))=\OM(y)+t\right).
	\end{align}
	\item It also holds that 
	$$\tilde{h}(\OM(x))\ge \tilde{h}(\OM(y)).$$
	\end{enumerate}
	
\end{corollary}

\begin{proof}
	To see the first statement, we first regard the case that $r'=0$. In this case, we have $r\le \OM(y)-\OM(x)$, so that the probability that $\OM(\flip_r(x))=\OM(y)+t$ is zero for $t>0$. For $t=0$ the statement trivially holds, since the right-hand side of~\eqref{33eq} is equal to one. For $r'>0$ the first statement follows from Lemma~\ref{LEM:DOMI} and an induction over $\OM(y)-\OM(x)$.

To prove the second statement we first assume that $\OM(y)-\OM(x)=1$. 
Let $r$ be the value that maximizes $\E\left(\max\{\OM(\flip_{r}(y))-\OM(y),0\}\right)$. Using Lemma~\ref{lem:mono_h} we obtain 
\begin{align*}
		\tilde{h}(\OM(y)) 
& = \E\left(\max\{\OM(\flip_{r}(y))-\OM(y),0\}\right)\\
& \le \E\left(\max\{\OM(\flip_{r}(x))-\OM(x),0\}\right) 
	\le \tilde{h}(\OM(x)).
\end{align*}
The general statement now follows by induction over $\OM(y)-\OM(x)$.
  
\end{proof}

We are now ready to prove the main result of this subsection, Lemma~\ref{lem:compare_h}.

\begin{proof}[Proof of Lemma~\ref{lem:compare_h}]
	Let $k\le n/2-n^{0.6}$, $x$ a search point with $d(x)\ge k$ and let $r \in [0..n]$. Since all random variables 
	$d(\flip_r(x))$, 
	$d(\flip_{n-r}(x))$, 
	$d(\flip_r(\bar{x}))$, and 
	$d(\flip_{n-r}(\bar{x}))$ are identically distributed, we can assume without loss of generality that 
	$\OM(x)\ge n/2$ and that $r\le n/2$. 
	
	Using this and the observations made in Remark~\ref{rem:expectedflips}, we easily see that 
	\begin{align*}
	\E(\OM(\flip_r(x)))
	& =		\OM(x)-r \tfrac{\OM(x)}{n}+r\tfrac{n-\OM(x)}{n}
	=		\OM(x)(1-r/n)+(n-\OM(x))r/n 
	\ge n/2.
	\end{align*}
	This shows that $\E(\OM(\flip_r(x)))-n^{0.6} \ge n/2-n^{0.6}\ge k$. Together with a Chernoff bound applied to $\OM(\flip_r(x))$ we thus obtain 
	\begin{align}\label{eq:3345}
	\Pr\left(\OM(\flip_r(x))<k \right)
	& \le\Pr\left(\OM(\flip_r(x))<\E(\OM(\flip_r(x)))-n^{0.6}\right)=\exp(-\Omega(n^{0.2})).
	\end{align}
	We first aim at bounding $\hat{h}(k)$. To this end, we use the estimate~\ref{eq:3345} to obtain 
	\begin{align*}
	& \E\left(k-\min\{k,d(\flip_{r}(x))\}\right)\\
	& \quad =
	\E\left(\OM(\flip_r(x))-(n-k)\mid \OM(\flip_r(x))>n-k\right)\Pr\left(\OM(\flip_r(x))>n-k\right) \\
	&\quad\quad	+\E\left(k-\OM(\flip_r(x))\mid \OM(\flip_r(x))<k\right)\Pr\left(\OM(\flip_r(x))<k\right)\\
	& \quad	\le 
	\E\left(\max\{\OM(\flip_r(x))-(n-k),0\}\right)+n\exp(-\Omega(n^{0.2})).
	\end{align*}
	Let $x'$ be a search point with $\OM(x')=n-k$, and let $r'=\max\{r-(d(x)-d(x')),0\}$. According to Corollary~\ref{cor:optoperator} it holds for all $i \in \N$ that 
	\[
	\Pr(\OM(\flip_r(x))=n-k+i)\le \Pr(\OM(\flip_{r'}(x'))=n-k+i).
	\]
 Using that 
 $\E\left(\max\{\OM(\flip_r(x))-(n-k),0\}\right)=\sum_{i\ge 1} i\Pr(\OM(\flip_r(x))=n-k+i)$ we obtain
    \[
 \E\left(k-\min\{k,d(\flip_{r}(x))\}\right)\le 	\E\left(\max\{\OM(\flip_{r'}(x'))-(n-k),0\}\right)+n\exp(-\Omega(n^{0.2})).
    \]
Referring to the definition of $\hat{h}$ and $\tilde{h}$ in equations  \eqref{def:h_hat} and \eqref{def:h_tilde}, and using the symmetries mentioned in the beginning of this proof, we bound 
	\begin{eqnarray*}
	\hat{h}(k)&=&\max\{k-\E\left(\min\{k,d(\flip_{r}(x))\}\right)\mid r\in[0..n], x \in \{0,1\}^n \text{ with } d(x) \ge k \}\\
	&=&\max\{k-\E\left(\min\{k,d(\flip_{r}(x))\}\right)\mid r\in[0..n/2], x \in \{0,1\}^n \text{ with } n/2\le \OM(x)\le n- k\}\\
	&\le &\max\{\E\left(\max\{\OM(\flip_{r'}(x'))-(n-k),0\}\right) \mid r'\in[0..n/2], x' \in \{0,1\}^n \text{ with } \OM(x')= n- k\}
	\\&&+n\exp(-\Omega(n^{0.2}))\\
	& \le & 
	 \tilde{h}(k)+n\exp(-\Omega(n^{0.2})).
	\end{eqnarray*}
	According to the definition of $h(k)$ in Equation \eqref{def:h}, we obtain 
	\begin{eqnarray*}
	h(k)=\max\{\hat{h}(i) \mid i \in[0..k] \}\le \max\{\tilde{h}(i) \mid i \in[0..k] \}+n\exp(-\Omega(n^{0.2})=\tilde{h}(k)+n\exp(-\Omega(n^{0.2}),
	\end{eqnarray*}
	where the last equality uses the monotonicity of $\tilde{h}(k)$ with respect to $k$ shown in Corollary~\ref{cor:optoperator}.
	  \end{proof}

As we will see in the next subsection, the $n\exp(-\Omega(n^{0.2}))$ term in this bound accounts for an additive $O(1)$ error in the runtime estimate only.

\subsubsection{Proof of Theorem~\ref{thm:UBdriftmax}}

With Lemma~\ref{lem:compare_h} at hand, we are now ready to prove Theorem~\ref{thm:UBdriftmax}.

\begin{proof}[Proof of Theorem~\ref{thm:UBdriftmax}]
As mentioned above, we easily obtain from Theorem~\ref{dis_U} that 
	\begin{equation}\label{sum_tilde_h}
		E\left(T_{A^*}\mid x(0)\right)\le \sum_{x=1}^{n-\OM(x(0))}\frac{1}{\tilde{h}(x)}.
	\end{equation}
	According to Lemma~\ref{lem:compare_h} it holds that $0<h(x)-\tilde{h}(x)\le n\exp(-\Omega(n^{0.2}))$ for $x\le n/2-n^{0.6}$. Using again that flipping a single bit on a bit string with $x$ zeros gives an expected progress in the \OM-value of $x/n$, we recall that $h(x)\ge \tilde{h}(x)\ge x/n$ for all $x\in[n/2]$. Therefore, 
	\[
	\frac{1}{\tilde{h}(x)}-\frac{1}{h(x)}=\frac{h(x)-\tilde{h}(x)}{\tilde{h}(x)h(x)}\le \frac{n\exp(-\Omega(n^{0.2}))}{(x/n)(x/n)}\le n^3\exp(-\Omega(n^{0.2})) \text{ for all } 1 \le x\le s. 
	\]
	Replacing $\tilde{h}(x)$ by $h(x)$ in inequality~\eqref{sum_tilde_h} and pessimistically assuming $\OM(x(0))=0$ we thus obtain 
	\begin{eqnarray*}
		E\left(T_{A^*}\right)&\le& \sum_{x=1}^{s}\left(\frac{1}{h(x)}+n^3\exp(-\Omega(n^{0.2}))\right)+\sum_{x=s}^{n}\frac{1}{\tilde{h}(x)}=\sum_{x=1}^{s}\frac{1}{h(x)}+\sum_{x=s}^{n}\frac{1}{\tilde{h}(x)}+O(1).
	\end{eqnarray*}
	Using that $\tilde{h}(x)\ge x/n$ for all $0<x\le n$ and $\tilde{h}(x)\ge 2x-n$ for all $n/2<x\le n$, we conclude
	\[
	\sum_{x=s}^{n}\frac{1}{\tilde{h}(x)}\le \sum_{x=s}^{n/2+n^{0.6}}\frac{n}{x}+\sum_{x=n/2+n^{0.6}}^{n}\frac{1}{2x-n}\le \frac{2n^{1.6}}{s}+\frac{n/2}{2n^{0.6}}=\Theta(n^{0.6}).
	\]
  \end{proof}

Theorem~\ref{thm:main22} follows from Theorems~\ref{thm:LBany} and~\ref{thm:UBdriftmax} by observing that $\Theta(n^{0.6})=o(n^{2/3}\ln^9(n))$.
\section{Fitness-Dependent Mutation Strength}
\label{sec:fitness-dependent}\label{SEC:FITNESS-DEPENDENT}

In the previous section we have seen that in order to compute the expected runtime of a best possible unary unbiased black-box algorithm for \onemax we can regard the algorithm $A^*$ that maximizes at any point in time the fitness drift. By Theorems~\ref{thm:LBany} and~\ref{thm:UBdriftmax} this algorithm cannot be worse (in expectation) than an optimal unary unbiased one by more than an additive $\Theta(n^{2/3}\ln^9(n))$ term. 

In this section we give a relatively concise description of $A^*$, i.e., we compute approximately the number of bits that need to be flipped in order to maximize the fitness drift. Since we are here talking about the drift in the fitness, it will be convenient to denote in this section by $d(x)=n-\OM(x)$ the fitness distance to the target. We also denote by
 \begin{equation}
 	R_{\opt}(d,n):=\min \left\{ \arg\max \E_{y \assign \flip_r(x)}\big(\max\{ \OM(y)-\OM(x),0 \}\big) \mid r \in [0..n],\OM(x)=n-d \right\},\label{def_R_dis}
 \end{equation}
the number of bits that need to be flipped in a search point $x\in\{0,1\}^n$ with $\OM(x)=n-d$ such that the expected drift $\E(\max\{0,\OM(\flip_{R_{\opt}(d,n)}(x))-\OM(x)\})$  is maximized (breaking ties by flipping fewer bits).  

The exact analysis of $R_{\opt}$ is rather tedious, as we will demonstrate below. Luckily, it turns out that we can safely approximate this point-wise drift maximizing function $R_{\opt}(d,n)$ by some a function $\Rapp:[0,1]\to [0..n]$ which maps the relative fitness distance $d(x)/n$ to a mutation strength. Since $\Rapp$ is much easier to work with, this is the focus of Section~\ref{sec:driftapprox}. For the approximation $\Rapp$ we make use of the fact that for values of~$r$ that are reasonably small compared to the problem dimension $n$ and the current fitness distance $d(x)$, the expected drift 
$\E(\max\{0,\OM(\flip_{R_{\opt}(d,n)}(x))-\OM(x)\})$ 
is almost determined by the relative fitness distance $d(x)/n$. For very small $d(x)=o(n)$ the fitness drift of flipping $R_{\opt}(d/n)=1$ bit is exactly $d(x)/n$, without any estimation error. We will also see in Section~\ref{sec:exact_drift} that it suffice to regard constant values~$r$. 

Once the approximation of the function $R_{\opt}$ by $\Rapp$ is established, we demonstrate in Section~\ref{sec:Ropt} a few properties of these two functions that will be useful in our subsequent computations; in particular for the numerical approximation of $\Rapp$, which is carried out in Section~\ref{sec:numerics}. Most importantly, we shall see that $\Rapp$ is monotone, i.e., the number of bits to flip in order to maximize the approximated point-wise drift decreases with increasing fitness. We also show that both $R_{\opt}$ and $\Rapp$ take only odd values, implying that flipping an even number of bits is suboptimal in all stages of the optimization process. 


To ease the computation of $\tilde{R}_{\opt}$, we analyze in detail the mutation rate for search points $x(t)$ with fitness distance $X_t\le (1/2-\eps)n$, where the constant $\eps$ satisfies $0<\eps<1/2$. Notice that by selecting between parent and its offspring at the end of each iteration in Algorithm~\ref{alg:algo} we have $\OM(x(t))=n-X_t$. For the remaining fitness distances, we simply take $\Rapp(1/2-\eps)$ for all $(1/2-\eps)n<X_t\le n/2$ and $n$ for all $X_t>n/2$. A detailed definition of $\Rappe$ is provided in equation \eqref{def_Rapp}. We will prove in Theorem~\ref{thm:approx_h} that our adhoc definition of $\Rappe(p)$ for $p \ge 1/2-\eps$ only causes an error term of $O(\eps n)$ in the runtime. 

\subsection{The Exact Fitness Drift \texorpdfstring{$B(n,d,r)$}{B(n,d,r)}} \label{sec:exact_drift}

In this subsection, we compute the exact fitness gain obtained from flipping $r$ bits. We shall then argue that once we have a fitness of at least $(\frac 12 + \eps) n$ for some constant $\eps$, the maximal fitness drift stems from flipping some constant number of bits.

Let $x$ be a binary string of length $n$ with fitness distance $d$, that is, with \onemax-value $n-d$. By the symmetry of the \onemax function, the expected progress of flipping $r$ bits in $x$ does not depend on the structure of $x$ but only on its fitness. We can therefore define the expected fitness gain from flipping $r$ random bits in $x$ by
	\[
B(n,d,r):=\E\left(\max\{0,d-d(\flip_r(x))\} \mid \OM(x)=n-d \right).
	\]

To compute $B(n,d,r)$ arithmetically, let us assume that $i \in [0..r]$ is the number of bits flipped from $0$ to $1$. Then $r-i$ bits have flipped in the opposite direction from $1$ to $0$, resulting in a progress of $i-(r-i)$. This progress is positive if and only if $i>r-i$, i.e., if and only if $i>r/2$. 
The probability for $i$ bits flipping in the ''good'' direction is $\binom{d}{i}\binom{n-d}{r-i}/\binom{n}{r}$. 
We therefore obtain
\begin{eqnarray*}
	B(n,d,r)=\sum_{i=\lceil r/2 \rceil}^{r}
	\frac{\binom{d}{i}\binom{n-d}{r-i}\left(2i-r\right)}{\binom{n}{r}}.
	\end{eqnarray*}

We show that the maximal fitness drift is obtained from flipping a constant number of bits once we have a fitness of at least $(\frac 12 + \eps)n$. The main argument is that flipping a single random bit already gives a better expected fitness gain than flipping many bits, which is due to the fact that when flipping many bits, the strong concentration of the hypergeometric distributions renders it highly unlikely that a fitness gain is obtained at all.

\begin{lemma}\label{lem:constantr}
	Let $0<\eps<1/2$ and $\alpha:=2\log(4/(\eps^2(1/2-\eps))$. Then for all $n\in\N$, $d\le(1/2-\eps)n$, and all $r\ge 2\alpha/\eps^2$, we have $B(n,d,r)<B(n,d,1)/2$. 
\end{lemma}

\begin{proof}
	Let $Z$ denote the number of ``good'' flips (i.e., the number of bits flipping from $0$ to $1$). As discussed in Remark~\ref{rem:expectedflips}, the random variable $Z$ follows a hypergeometric distribution with mean value $\E(Z)=dr/n=pr<r/2$, where we abbreviate $p:=d/n$. Applying the Chernoff bound presented in Theorem 1.9 (b) in~\cite{Doerr11bookchapter} to $Z$, we obtain 
	\begin{eqnarray}
	&&\Pr\left(Z>r/2\right)
	=\Pr\left(Z> \E(Z)/(2p)\right)
	\le\left(\frac{e^{1/(2p)-1}}{(1/(2p))^{1/(2p)}}\right)^{\E(Z)}\nonumber\\
	&=&\left((2p)^{1/(2p)} e^{1/(2p)-1}\right)^{rp}
	=\left(\frac{(2pe)^{1/(2p)}}{e}\right)^{rp}
	=\left(\frac{2pe}{e^{2p}}\right)^{r/2}.\label{eq_const_r}
	\end{eqnarray}
	We then regard $B(n,d,r)/(B,n,d,1)\le r\Pr(Z>r/2)/(d/n)=(r/p)\left(\frac{2pe}{e^{2p}}\right)^{r/2}=r(2e)^{r/2}p^{r/2-1}e^{-pr}$. We notice that for fixed $r\ge 2\alpha/\eps^2$ and $0<p<1/2-\eps$,
	\begin{align*}
		(p^{r/2-1}e^{-pr})'&=(r/2-1)p^{r/2-2}e^{-pr}-rp^{r/2-1}e^{-pr}=p^{r/2-2}e^{-pr}(r/2-1-rp)\\
		&\ge p^{r/2-2}e^{-pr}(2\alpha/\eps-1)>0,
	\end{align*} 
	thus it remains to check the statement for $d=(1/2-\eps)n$.
	Using the Taylor expansion $e^{2\delta}=1+2\delta+2\delta^2+(4/3)\delta^3+O(\delta^4)<1+2\delta+3\delta^2$ we see that $\lim_{p\to 1/2-\eps} 2pe/e^{2p}=\lim_{\delta\to\eps}2(1/2-\delta)e/e^{2(1/2-\delta)}=\lim_{\delta\to\eps}(1-2\delta)e^{2\delta}<\lim_{\delta\to\eps}(1-2\delta)(1+2\delta+3\delta^2)=1-\eps^2-6\eps^3<1-\eps^2$.
	Therefore we obtain $B(n,d,r)\le r\Pr(Z>r/2)\le r(2pe/e^{2p})^{r/2}< r(1-\eps^2)^{r/2}$. Using the fact that $(1-\eps^2)^{1/\eps^2}<1/e$, $r\ge 2\alpha/\eps^2>2/\eps^2$, and $r(1-\eps^2)^{r/2}$ monotonically decreases when $r>2/\eps^2$, we obtain $B(n,d,r)< r\exp(-\alpha)\le 2\alpha \exp(-\alpha)/\eps^2$. Since $\log(\alpha\exp(-\alpha))=\log(\alpha)-\alpha<-\alpha/2$, then 
	$B(n,d,r)<2\alpha \exp(-\alpha)/\eps^2<2\exp(-\alpha/2)/\eps^2=(2/\eps^2)(\eps^2(1/2-\eps)/4)=(1/2-\eps)/2=B(n,d,1)/2$ for $d=(1/2-\eps)n$.
	 \end{proof}

\subsection{Approximating \texorpdfstring{$B(n,d,r)$}{B(n,d,r)} via \texorpdfstring{$A(r,\frac dn,1-\frac dn)$}{A(r,.,.)}} \label{sec:driftapprox}

When $n$ and $d$ are large compared to $r$, the expected progress $B(n,d,r)$ is almost determined by $d/n$. This inspires the following definition of $A(r,p,q)$ which will have the property that $A(r,\frac dn,1-\frac dn)$ is a good approximation of $B(n,d,r)$. The definition for general $p$ and $q$ instead of $p$ and $q = 1-p$ will be useful in the following proofs. 
\begin{definition}
	For all $r\in\mathbb{N}$, $p\in[0,1]$, and $q\in[0,1]$, let
		\begin{equation}\label{defineA}
		A(r,p,q):=\sum_{i=\lceil r/2 \rceil}^{r} \binom{r}{i}(2i-r)p^iq^{r-i}.
		\end{equation}
\end{definition}

The following Theorem~\ref{THM:APPROXBA} makes precise how well for $p=d/n$ and $q=1-p$ the value $A(r,p,q)$  approximates the expected progress $B(n,d,r)$. 

\begin{theorem}
	\label{thm:approxBA}\label{THM:APPROXBA}
	Let $0<\eps<1/2$ and  $\alpha=2\log(4/(\eps^2(1/2-\eps))$ (as in Lemma~\ref{lem:constantr}). Then for all $n\in \N$ large enough, all $r<2\alpha/\eps^2$, and all $2r \le d\le (1/2-\eps)n$, we have
	\begin{equation}\label{o1}
	\left|A\left(r,\frac{d}{n},\frac{n-d}{n}\right)-B(n,d,r)\right|< \frac{3r^3}{d}.
	\end{equation}
\end{theorem}

The first step in the proof of Theorem~\ref{THM:APPROXBA} is the following statement, which compares suitable $A$-values with $B$-values. Note that here we profit from the general definition of $A(r,p,q)$ instead of the special case $A(r,p,1-p)$.

\begin{lemma} \label{lem:HilfeBA}
	Let $n\in \N$ be sufficiently large and $1\le r\le d\le (1/2-\eps)n$ with $0<\eps<1/2$. 
	It holds that 
	\begin{equation} \label{AB}
	A\left(r,\frac{d}{n},\frac{n-d}{n-r}\right)\ge B(n,d,r) \ge A\left(r,\frac{d-r}{n},\frac{n-d-r}{n}\right).
	\end{equation}
\end{lemma}

\begin{proof}
	For any two positive integers $r$ and $i \leq r$ we abbreviate $$(r)_i:=r(r-1)\ldots(r-i+1)=\prod_{j=0}^{i-1}{(r-j)}.$$
	With this notation, we can express $B(n,m,r)$ as
	\begin{eqnarray*}
		B(n,d,r) 
		=
		\sum_{i=\lceil  r/2 \rceil}^{r}
		\frac{\binom{d}{i}\binom{n-d}{r-i}\left(2i-r\right)}{\binom{n}{r}}
		= 
		\sum_{i=\lceil r/2 \rceil}^{r}
		\frac{(d)_i(n-d)_{r-i}}{(n)_r}\binom{r}{i}\left(2i-r\right).
	\end{eqnarray*}
	From the elementary fact that for all $\lceil r/2 \rceil\le i\le r$, we have $(n-d)_{r-i}\le (n-d)^{r-i}$ and $(n)_r\ge(n)_i(n-r)^{r-i}$, we obtain
	\begin{eqnarray*}
		\frac{(d)_i(n-d)_{r-i}}{(n)_r} 
		\le \frac{(d)_i(n-d)_{r-i}}{(n)_i(n-r)^{r-i}}
		\le \left(\frac{d}{n}\right)^i\left(\frac{n-d}{n-r}\right)^{r-i}.
	\end{eqnarray*}
	This shows $B(n,d,r) \le A\left(r,\frac{d}{n},\frac{n-d}{n-r}\right)$.
	
	To show the second inequality, we use the fact that for all $i\le r\le n$, we have $(n)_r\le n^r$, $(d)_i\ge (d-r)^i$, and $(n-d)_{r-i}\ge (n-d-r)^{r-i}$. Consequently,
	\begin{eqnarray*}
		\frac{(d)_i(n-d)_{r-i}}{(n)_r} \ge\frac{(d-r)^i(n-d-r)^{r-i}}{n^r} 
		= \left(\frac{d-r}{n}\right)^i\left(\frac{n-d-r}{n}\right)^{r-i},
	\end{eqnarray*}
	yielding $B(n,d,r) \ge A\left(r,\frac{d-r}{n},\frac{n-d-r}{n}\right)$.
	  \end{proof}

With Lemma~\ref{lem:HilfeBA} at hand, we now prove Theorem~\ref{THM:APPROXBA}.

\begin{proof}[Proof of Theorem~\ref{THM:APPROXBA}]
	According to the definition of $A(r,p,q)$ in~\eqref{defineA} we have 
	\begin{equation*}
	\frac{\partial A(r,p,q)}{\partial q}=\sum_{i=\lceil r/2 \rceil}^{r} (r-i)\binom{r}{i}(2i-r)p^iq^{r-i-1}<\frac{r}{2}\sum_{i=\lceil r/2 \rceil}^{r} \binom{r}{i}(2i-r)p^iq^{r-i-1}=\frac{rA(r,p,q)}{2q}<\frac{r^2}{2q},
	\end{equation*}
	where we have used in the last step that $A(r,p,q)\le r$. 
	
	Using that $\frac{n-d}{n}>1/2$ and $0<\frac{n-d}{n-r}-\frac{n-d}{n}=\frac{n-d}{n-r}\cdot\frac{r}{n}<(1+o(1))\frac{r}{n}$, we bound 
	\begin{eqnarray*}
		&&A\left(r,\frac{d}{n},\frac{n-d}{n-r}\right)- A\left(r,\frac{d}{n},\frac{n-d}{n}\right)\le  \frac{r^2}{2(1/2)}\left(\frac{n-d}{n-r}-\frac{n-d}{n}\right)\le \frac{(1+o(1))r^3}{n}.
	\end{eqnarray*}
	Similarly we have 
	\begin{equation*}
	\frac{\partial A(r,p,q)}{\partial p}=\sum_{i=\lceil r/2 \rceil}^{r} i\binom{r}{i}(2i-r)p^{i-1}q^{r-i}<r\sum_{i=\lceil r/2 \rceil}^{r} \binom{r}{i}(2i-r)p^{i-1}q^{r-i}=\frac{rA(r,p,q)}{p}<\frac{r^2}{p}.
	\end{equation*}
	Using $d\ge 2r$ we obtain
	\begin{equation*}
	A\left(r,\frac{d}{n},\frac{n-d-r}{n}\right)-A\left(r,\frac{d-r}{n},\frac{n-d-r}{n}\right)\le \frac{r^2}{(d-r)/n}\cdot\frac{r}{n}=\frac{r^3}{d-r}>\frac{2r^3}{d}.
	\end{equation*}
	Therefore
	\[
	A\left(r,\frac{d}{n},\frac{n-d}{n-r}\right)- A\left(r,\frac{d-r}{n},\frac{n-d-r}{n}\right)=\frac{(1+o(1))r^3}{n}+\frac{2r^3}{d}<\frac{3r^3}{d}.
	\]
	By Lemma~\ref{lem:HilfeBA}, it suffices to estimate $|A(r,\frac{d}{n},\frac{n-d}{n})-B(n,d,r)|<A(r,\frac{d}{n},\frac{n-d}{n-r})-A(r,\frac{d-r}{n},\frac{n-d-r}{n})<\frac{3r^3}{d}$.
	  \end{proof}

\subsection{Approximate Optimal Number of Bits to Flip }\label{sec:Ropt}\label{SEC:ROPT}

The goal of this section is to approximate the function $R_{\opt}$ which tells us how many bits one should flip in order to maximize the point-wise drift. Given Theorem~\ref{THM:APPROXBA} above, it is tempting to assume that the map $p \mapsto \arg \max_{r\in{\N}} A(r,p,1-p)$ should do. Analogous to Lemma~\ref{lem:constantr} we show in the following lemma that it suffices to regard constant $r$ for the approximated drift.

\begin{lemma}\label{lem:constantr_approx}
	Let $0<\eps<1/2$ and $\alpha:=2\log(4/(\eps^2(1/2-\eps))$. For all $n\in\N$ with $d\le (1/2-\eps)n$ and all $r\ge 2\alpha/\eps^2$, the expected approximated drift $A(r,\frac{d}{n},\frac{n-d}{n})<A(1,\frac{d}{n},\frac{n-d}{n})/2$. 
\end{lemma}

\begin{proof}
	Consider the binomial random variable $Z\sim \Bin(r,d/n)$. Let $p=d/n$ then $\E(Z)=pr$. Applying the Chernoff bound presented in Theorem 1.9 (b) in~\cite{Doerr11bookchapter} to $Z$, we obtain
	\begin{eqnarray*}
	&&\Pr\left(Z>r/2\right)
	=\Pr\left(Z> \E(Z)/(2p)\right)
	\le\left(\frac{e^{1/(2p)-1}}{(1/(2p))^{1/(2p)}}\right)^{\E(Z)}=\left(\frac{2pe}{e^{2p}}\right)^{r/2},
	\end{eqnarray*}
	which is the same inequality as \eqref{eq_const_r} in Lemma~\ref{lem:constantr}. Since
	$A(r,p,1-p):=\sum_{i=\lceil r/2 \rceil}^{r} \binom{r}{i}(2i-r)p^i(1-p)^{r-i}\le r\sum_{i=0}^{r}\mathds{1}_{2i>r}\binom{r}{i}p^i(1-p)^{r-i}=r\Pr(Z>r/2)$ and $A(1,p,1-p)=p$,
	with the same arguments as in Lemma~\ref{lem:constantr}, the statement holds.
\end{proof}

In the remainder of this section we show that flipping a number $r$ of bits that maximizes $A(r,d/n,1-d/n)$ yields indeed a good approximation of the best possible expected progress. Since in principle there could be more than one $r$ maximizing $A(r,p,1-p)$ for a given relative distance $p \in (0,1/2)$, we break ties by preferring smaller values of $r$. For $p \ge 1/2 - \eps$, where our reasoning above was not applicable, we do not try to find an optimal number of bits to flip, but rather one that does the job of giving a near-optimal runtime. Since a random initial search point has a fitness close to $n/2$, not too much time is spent in this regime anyway. Consequently, we define, for all $\eps>0$, 
\begin{align}\label{def_Rapp}
\Rappe(p):= 
\begin{cases}
&	\min\big\{\arg \max_{r\in{\N}} A(r,p,1-p)\big\} \text{ for } 0<p\le 1/2-\epsilon,\\
& \Rapp \left(1/2-\epsilon\right) \text{ for } 1/2-\epsilon<p\le 1/2,\\
& n \text{ for } p>1/2,
\end{cases}
\end{align}
According to Lemma~\ref{lem:constantr_approx} the function $\Rappe$ is well defined (for all $\eps>0$).

We prove two important properties of the functions $\Rappe$, which are summarized in the following theorem.

\begin{theorem}
\label{mathcalR}\label{thm:monotonicity}\label{THM:MONOTONICITY}
For all $\eps>0$ the function $\Rappe$ is monotonically increasing with respect to $p$. For all $d\le n/2$, 
$\Rappe(d/n)$ and $R_{\opt}(d,n)$ are odd values.
\end{theorem}

The proof of the second claim in Theorem~\ref{thm:monotonicity} will be carried out in Section~\ref{sec:oddness}. It is purely combinatorial. The proof of the monotonicity of $\Rappe$, in contrast, is surprisingly technical. It will be carried out in Section~\ref{subsub:monotone}.

\subsubsection{\texorpdfstring{$R_{\opt}$}{Ropt} and \texorpdfstring{$\Rapp$}{Rtilde} Attain Only Odd Values}
\label{sec:oddness}

One possibly surprising property of the functions $R_{\opt}$ and $\Rappe$ is that they take only odd values. That is, regardless of how far we are from the optimum, the maximal drift is obtained for an odd number of bit flips. The following two lemmas show this statement for the approximate and the exact drift, respectively. 
 
\begin{lemma}[flipping even numbers of bits is sub-optimal, statement for the approximated drift $A$]
\label{A2k}\label{lem:odd}\label{LEM:ODD}
	For all $k\in\mathbb{N}$ and $p \in (0,1)$ it holds that 
	$\frac{A(2k,p,1-p)}{2k}=\frac{A(2k+1,p,1-p)}{2k+1}$. Consequently $\Rappe(d/n)$ takes odd values for all $d\le n/2$ and all $\eps>0$.
\end{lemma}

\begin{proof}
	By definition of the function $A$ in~\eqref{defineA} and using the facts that for all $r \in \N$ and all $i\le r$ we have
	\begin{equation*}
		\binom{r}{i}=\binom{r}{r-i},~~~\binom{r}{i}i=r\binom{r-1}{i-1}, \text{ and } \binom{r}{i}(r-i)=r\binom{r-1}{i},
	\end{equation*}
	we easily see that
	\begin{eqnarray}
		A(r,p,q)&=&\sum_{i=\lceil r/2 \rceil}^{r} \binom{r}{i}(2i-r)p^iq^{r-i}\nonumber\\
		&=&\sum_{i=\lceil r/2  \rceil}^{r}\left(\binom{r}{i} i-\binom{r}{i}(r-i)\right)p^iq^{r-i}\nonumber\\
		&=&r \sum_{i=\lceil r/2  \rceil}^{r}\left(\binom{r-1}{i-1}-\binom{r-1}{i}\right)p^iq^{r-i}.\label{A(2k+1)}
	\end{eqnarray}
	This shows that, for all $k \in \N$, 
	\begin{eqnarray*}
		&&\frac{A(2k+1,p,q)}{2k+1}=\sum_{i=k+1}^{2k+1}\left(\binom{2k}{i-1}-\binom{2k}{i}\right)p^iq^{2k+1-i}\\
		&=&\sum_{i=k+1}^{2k+1}
		\left(\binom{2k-1}{i-1}-\binom{2k-1}{i}+\binom{2k-1}{i-2}-\binom{2k-1}{i-1}\right)p^{i}q^{2k+1-i}\\
		&=&\sum_{i=k+1}^{2k}\left(\binom{2k-1}{i-1}-\binom{2k-1}{i}\right)p^{i}q^{2k+1-i}+\sum_{i=k+1}^{2k+1}\left(\binom{2k-1}{i-2}-\binom{2k-1}{i-1}\right)p^{i}q^{2k+1-i}\\
		&=&\sum_{i=k+1}^{2k}\left(\binom{2k-1}{i-1}-\binom{2k-1}{i}\right)p^{i}q^{2k+1-i}+\sum_{i=k}^{2k}\left(\binom{2k-1}{i-1}-\binom{2k-1}{i}\right)p^{i+1}q^{2k-i}\\
		&=&\sum_{i=k}^{2k}\left(\binom{2k-1}{i-1}-\binom{2k-1}{i}\right)\left(p^{i}q^{2k+1-i}+p^{i+1}q^{2k-i}\right)-\left(\binom{2k-1}{k-1}-\binom{2k-1}{k}\right)p^{k}q^{k+1}\\
		&=&\sum_{i=k}^{2k}\left(\binom{2k-1}{i-1}-\binom{2k-1}{i}\right)p^{i}q^{2k-i}=\frac{A(2k,p,q)}{2k},
	\end{eqnarray*}
	where we have used in the last step that $p+q=1$ and $\binom{2k-1}{k-1} = \binom{2k-1}{k}$.
  \end{proof}

Lemma~\ref{lem:odd} is not an artifact of the approximation of the drift by function $A$ but also holds for the exact drift-maximizing function $B$. This lemma will not be needed in the following, but we believe it to be interesting in its own right. The reader only interested in the proof of the main results of this work can skip this proof.

\begin{lemma}[flipping even numbers of bits is sub-optimal, statement for exact drift $B$]
	\label{B}
	For all $n,d,k\in \mathbb{N}$ satisfying  
	$0<d\le\frac{n}{2}$  
	and $0< 2k+1 \le n$, 
	it holds that $B(n,d,2k)<{B(n,d,2k+1)}$. 
	Moreover,
	$\frac{B(n,d,2k)}{2k}=\frac{B(n,d,2k+1)}{2k+1}$ holds.
\end{lemma}

\begin{proof}
	Using again the shorthand $(r)_i:=r(r-1)\cdots(r-i+1)$ for all positive integers $r$ and $i \leq r$,
	we get
	\begin{eqnarray*}
		B(n,d,2k+1)&=&\frac{1}{\binom{n}{2k+1}}\sum_{i=k+1}^{2k+1} \binom{d}{i}\binom{n-d}{2k-i+1}(2i-2k-1)\\
		&=&\frac{1}{\binom{n}{2k+1}}\sum_{i=0}^{k} \binom{d}{i+k+1}\binom{n-d}{k-i}(2i+1)\\
		&=&\frac{(n-2k-1)!}{n!}\sum_{i=0}^{k} \binom{2k+1}{k-i}(d)_{i+k+1}(n-d)_{k-i}(2i+1).
	\end{eqnarray*}
	Similarly, we obtain
	\begin{eqnarray*}
		B(n,d,2k)&=&\frac{1}{\binom{n}{2k}}\sum_{i=k+1}^{2k} \binom{d}{i}\binom{n-d}{2k-i}(2i-2k)\\
		&=&\frac{1}{\binom{n}{2k}}\sum_{i=0}^{k-1} \binom{d}{i+k+1}\binom{n-d}{k-i-1}(2i+2)\\
		&=&\frac{(n-2k)!}{n!}\sum_{i=0}^{k-1}\binom{2k}{k-i-1} (d)_{i+k+1}(n-d)_{k-i-1}(2i+2).
	\end{eqnarray*}
	Therefore, the ratio of $B(n,d,2k)$ and $B(n,d,2k+1)$ is
	\begin{eqnarray*}
		\frac{B(n,d,2k)}{B(n,d,2k+1)}&=&
		\frac{(n-2k)\sum_{i=0}^{k-1}\binom{2k}{k-i-1} (d)_{i+k+1}(n-d)_{k-i-1}(2i+2)}
		{\sum_{i=0}^{k} \binom{2k+1}{k-i}(d)_{i+k+1}(n-d)_{k-i}(2i+1)}\\
		&=&
		\frac{(n-2k)\sum_{i=0}^{k-1}\binom{2k}{k-i-1} (d-k-1)_{i}(n-d)_{k-i-1}(2i+2)}
		{\sum_{i=0}^{k} \binom{2k+1}{k-i}(d-k-1)_{i}(n-d)_{k-i}(2i+1)}.
	\end{eqnarray*}
	Replacing $(d-k)$ with $u$ and $(n-d)$ with $v$ gives
	\begin{eqnarray*}
		\frac{B(n,d,2k)}{B(n,d,2k+1)}&=&
		\frac{(u+v-k)\sum_{i=0}^{k-1}\binom{2k}{k-i-1} (u-1)_{i}(v)_{k-i-1}(2i+2)}
		{\sum_{i=0}^{k} \binom{2k+1}{k-i}(u-1)_{i}(v)_{k-i}(2i+1)}.
	\end{eqnarray*}
	
	Using the shorthand $\lambda_i^j$ for the coefficient of $r^j$ in the polynomial $(r)_i$, we now take a close look at the denominator, which is a polynomial in $u$ and $v$. 
	It is not difficult to see that for all $a,b\in\mathbb{N} \cup \{0\}$, the coefficient of the term
	$u^av^b$ in the denominator equals
	\begin{eqnarray*}
		\psi(a,b)&=&\sum_{i=0}^{k}\binom{2k+1}{k-i}
		\lambda^{a+1}_{i+1}\lambda^{b}_{k-i}
		(2i+1),
	\end{eqnarray*}
	while the coefficient of term $u^av^b$ in numerator equals
	\begin{eqnarray*}
		\phi(a,b)&=&\sum_{i=0}^{k-1}\binom{2k}{k-i-1}
		\left(\lambda^a_{i+1}\lambda^b_{k-i-1}+
		\lambda^{a+1}_{i+1}\lambda^{b-1}_{k-i-1}-
		k\lambda^{a+1}_{i+1}\lambda^{b}_{k-i-1}
		\right)
		(2i+2).
	\end{eqnarray*}
	Since $(r)_{i+1}=(r-i)(r)_{i}$, it is easily verified that
	\begin{eqnarray}\label{lambda}
	\lambda_i^j-i\lambda_{i}^{j+1}=\lambda_{i+1}^{j+1}.
	\end{eqnarray}
	We use \eqref{lambda} to simplify $\phi(a,b)$ in the following way.
	\begin{eqnarray*}
		\phi(a,b)&=&\sum_{i=0}^{k-1}\binom{2k}{k-i-1}
		\left(\lambda^{a+1}_{i+2}\lambda^b_{k-i-1}+
		\lambda^{a+1}_{i+1}\lambda^{b}_{k-i}
		\right)
		(2i+2)\\
		&=&\sum_{i=0}^{k}\left(\binom{2k}{k-i-1}(2i+2)+\binom{2k}{k-i}(2i)\right)\lambda^{a+1}_{i+1}\lambda^{b}_{k-i}\\
		&=&\sum_{i=0}^{k}\left((2i+2)+\frac{k+i+1}{k-i}(2i)\right)\binom{2k}{k-i-1}\lambda^{a+1}_{i+1}\lambda^{b}_{k-i}\\
		&=&\sum_{i=0}^{k}\frac{2k(2i+1)}{k-i}\binom{2k}{k-i-1}\lambda^{a+1}_{i+1}\lambda^{b}_{k-i}\\
		&=&\frac{2k}{2k+1}\sum_{i=0}^{k}\binom{2k+1}{k-i}\lambda^{a+1}_{i+1}\lambda^{b}_{k-i}(2i+1)\\
		&=&\frac{2k}{2k+1}\psi(a,b).
	\end{eqnarray*}
	The above holds for all $0\le a,b \le k$, showing that indeed
	\begin{eqnarray*}
		\frac{B(n,d,2k)}{B(n,d,2k+1)}=\frac{2k}{2k+1}.
	\end{eqnarray*}
	  \end{proof}

\subsubsection{Monotonicity of \texorpdfstring{$\Rappe$}{R}}
\label{subsub:monotone}
We now argue that, for all $\eps>0$, the function $\Rappe$ is monotone. It seems quite intuitive that the optimal number of bit flips should decrease with decreasing distance to the optimum, and this has been previously observed empirically, e.g., in~\cite{Back92,FialhoCSS08,FialhoCSS09}. However, formally proving the desired monotonic relationship requires substantial technical work. We note that, as a side result, Lemma~\ref{lem:R13} shows that for search points having a distance of less than $n/3$ to the optimum (or its complement), the maximal approximated fitness gain is obtained by 1-bit flips.  
\begin{lemma}[and definition of cut-off points]
\label{cmp}
\label{LEM:CUTOFF}
For any two integers $0\le k_1<k_2$, the functions $p \mapsto A(2k_1+1,p,1-p)$ and $p \mapsto A(2k_2+1,p,1-p)$ intersect exactly once in the interval $(0,1/2]$. Denoting this intersection $p_0$ and letting $A_0:=A(2k_1+1,p_0,1-p_0)$, we call $(p_0,A_0)$ the \textbf{\emph{cut-off point}} of $A(2k_1+1,p,1-p)$ and $A(2k_2+1,p,1-p)$.

We have $A(2k_1+1,p,1-p)>A(2k_2+1,p,1-p)$ if and only if $0<p<p_0$. 
\end{lemma}

\begin{figure}[t]
\begin{center}
\includegraphics[width=0.7\linewidth]{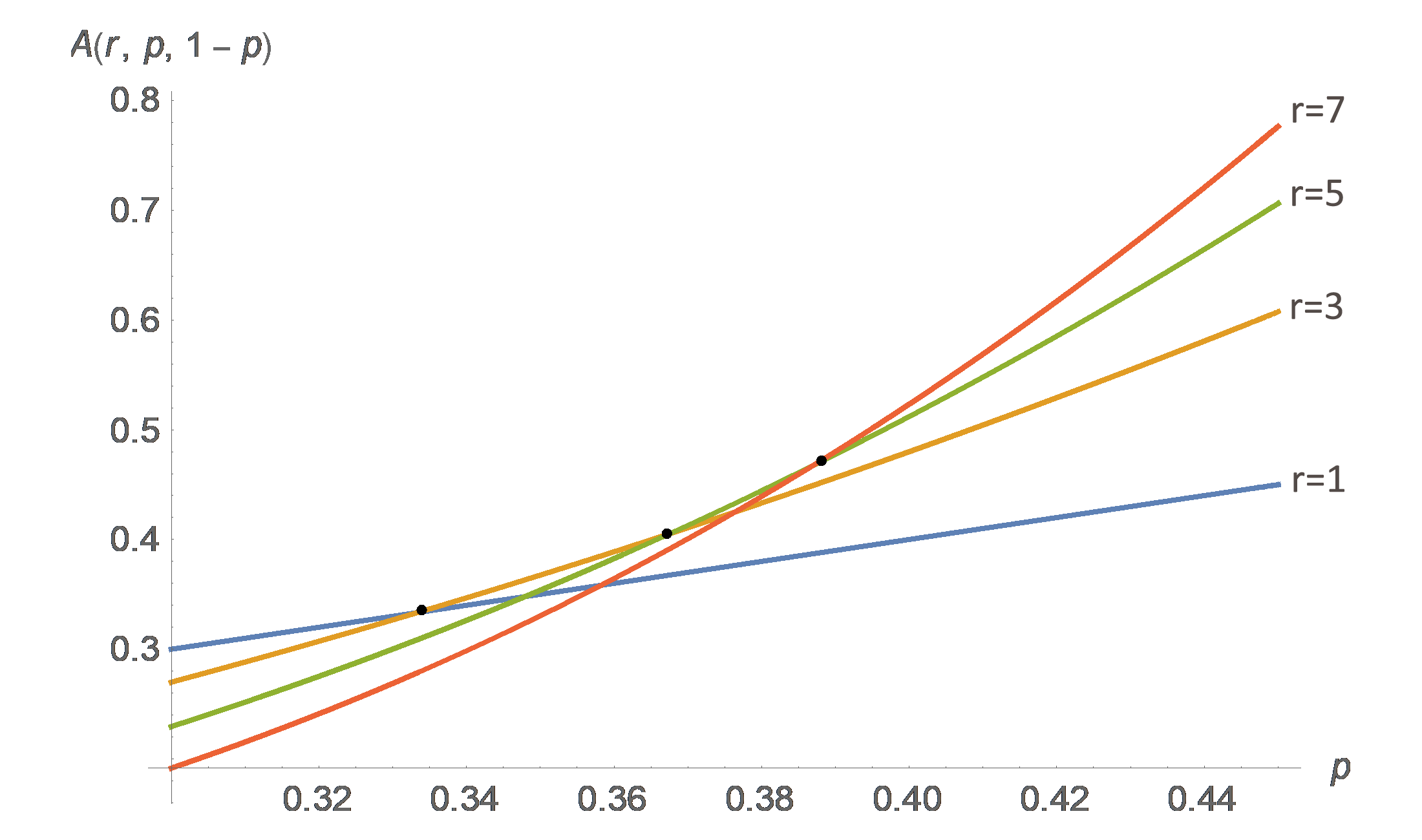}
\end{center}
\caption{The approximated drift $A(k,p,1-p)$ for $k=1,3,5,7$. By Lemma~\ref{lem:R13} the function $r \mapsto A(r,p,1-p)$ is maximized for $r=1$ whenever $p<1/3$.}
\label{fig:cutoff}
\end{figure}

The graph in Figure~\ref{fig:cutoff} illustrates the functions $p\mapsto A(k,p,1-p)$ for $k=1,3,5,7$.
The precise cut-off points will be computed numerically in Section~\ref{sec:numerics}. 

In order to prove Lemma~\ref{LEM:CUTOFF}, we first show the following combinatorial lemma.
\begin{lemma} \label{pqtoq}For all $k,r\in \mathbb{N}\cup\{0\}$ and $0\le q\le 1$ it holds that
	\begin{equation}\label{eq:helper111}
	\sum_{i=0}^{k}\binom{k+r+1}{i}(1-q)^{k-i}q^i=\sum_{i=0}^{k}\binom{i+r}{r}q^i.
	\end{equation}
\end{lemma}

\begin{proof}
	We prove the equation by induction. It is obvious that equation~\eqref{eq:helper111} holds for all $r\ge 0$ and $k=0$. Assume that it holds for some pair of integers $(k,r+1)$, then the following computation shows that it also holds for $(k+1,r)$. 
	\begin{eqnarray*}
		&& \sum_{i=0}^{k+1}\binom{k+1+r+1}{i}(1-q)^{k-i+1}q^i\\
		&=& \binom{k+r+2}{k+1}q^{k+1}+(1-q)\sum_{i=0}^{k}\binom{k+(r+1)+1}{i}(1-q)^{k-i}q^i\\
		&=& \binom{k+r+2}{r+1}q^{k+1}+(1-q)\sum_{i=0}^{k}\binom{i+r+1}{r+1}q^i\\
		&=& \sum_{i=0}^{k+1}\binom{i+r}{r}q^i.
	\end{eqnarray*}
	For arbitrary combinations of $k$ and $r$, we thus get the desired correctness of~\eqref{eq:helper111} for the pair $(k,r)$ inductively from that of the pair $(0,r+k)$.
  \end{proof}

We use Lemma~\ref{pqtoq} to compute the second derivative of $A(r,p,q)$ for $q$ in Lemma~\ref{lemma:ddq} and then obtain the second derivative of $A(r,p,q)$ for $p$ in Lemma~\ref{d2A}. We first notice that $A(1,p,1-p)=p$ and $\mathrm{d}A(1,p,1-p)/\mathrm{d}p=1$. Thus we only look at the second derivative for $r>1$.

\begin{lemma} \label{lemma:ddq}
For all $k\in\mathbb{N}$ 
and all $0<p\le 1/2$, it holds that
\begin{eqnarray}\label{ddq}
\frac{\mathrm{d}^2A(2k+1,p,1-p)}{(\mathrm{d}(1-p))^2}=c_k p^{k-1}(1-p)^{k-1},
\end{eqnarray}
	where $c_k$ is a constant related to $k$ via
	\begin{eqnarray*}
		c_k:=2(2k-1)(2k+1)\binom{2k-2}{k-1}=\frac{4k+2}{\beta(k,k)},
	\end{eqnarray*}
	and $\beta(x,y):=\int_{0}^1 t^{x-1}(1-t)^{y-1}dt=\frac{\Gamma(x)\Gamma(y)}{\Gamma(x+y)}$ is the well-known beta function. 
\end{lemma}

\begin{proof}
	Set $q:=1-p$.
	We expand $A(2k+1,p,q)$ according to Equation~\eqref{A(2k+1)} and use Lemma~\ref{pqtoq} to obtain the following
	\begin{eqnarray*} 
	A(2k+1,p,q)
	&=&(2k+1)\sum_{i=k}^{2k}
	\left(\binom{2k}{i}-\binom{2k}{i+1}\right)p^{i+1}q^{2k-i}\\
	&=&(2k+1)\sum_{i=0}^{k}
	\left(\binom{2k}{k+i}-\binom{2k}{k+i+1}\right)p^{k+i+1}q^{k-i}\\
	&=&(2k+1)p^{k+1}\cdot \sum_{i=0}^{k}
	\left(\binom{2k}{i}-\binom{2k}{i-1}\right)p^{k-i}q^{i}\\
	&=&(2k+1)p^{k+1}\left[ \sum_{i=0}^{k}
	\binom{k+(k-1)+1}{i}p^{k-i}q^{i}
	+q\sum_{i=0}^{k-1}
	\binom{(k-1)+k+1}{i}p^{k-1-i}q^{i}\right]\\
	&=&(2k+1)p^{k+1}\left[ \sum_{i=0}^{k}
	\binom{i+(k-1)}{i}q^{i}+ q\sum_{i=0}^{k-1}
	\binom{i+k}{i}q^{i}\right]\\
	&=&(2k+1)p^{k+1}\cdot \sum_{i=0}^{k}
	\left(\binom{k+i-1}{i}-\binom{k+i-1}{i-1}\right)q^{i}.\\
	\end{eqnarray*}
	We extract the term $p^{k+1}$ in the above equation and define the polynomial $f_k(q):=\sum_{i=0}^{\infty}a_k^{i}q^{i}$ with coefficient $a_k^{i}=0$ when $i>k$ and
	\begin{eqnarray}\label{aki}
	a_k^{i}&:=&\binom{k+i-1}{i}-\binom{k+i-1}{i-1}=\binom{k+i-1}{i}\frac{k-i}{k} \text{ when } i\le k.
	\end{eqnarray} 
	Then $A(2k+1,p,q)=(2k+1)p^{k+1}f_k(q)$. 
	We use the general Leibniz rule 
	for the second derivative (informally, this rule states that $(fg)^{''}=f''g+2f'g'+fg''$) and obtain
	\begin{eqnarray*}
	\frac{\mathrm{d^2}A(2k+1,p,q)}{\mathrm{d^2}q}=
		(2k+1)p^{k-1}\cdot \left( p^2f''_k(q)-2(k+1)pf'_k(q)+(k+1)kf_k(q)\right).
	\end{eqnarray*}
	It remains to prove that
	\begin{eqnarray}\label{deri}
	p^2f''_k(q)-2(k+1)pf'_k(q)+(k+1)kf_k(q)=2(2k-1)\binom{2k-2}{k-1}q^{k-1}.
	\end{eqnarray}
	We look at the coefficient of $q^i$ in the left part of equation~\eqref{deri} and we denote it by $c_k^i$.
	By expanding $f_k(q)$ into a polynomial and replacing $p$ by $1-q$, we see that $c_k^i$ equals the coefficient of $q^i$ in the following expression 
	\begin{eqnarray*}
		(1-q)^2(a_k^iq^i+a_k^{i+1}q^{i+1}+a_k^{i+2}q^{i+2})''
		-2(k+1)(1-q)(a_k^iq^i+a_k^{i+1}q^{i+1})'
		+(k+1)k(a_k^iq^i).
	\end{eqnarray*}
	This shows that $c_k^i$ is equal to 
	\begin{eqnarray*}
	a_k^i\left((k+1)k+2(k+1)i+i(i-1)\right)
	+ a_k^{i+1}\left(-2(k+1)(i+1)-2(i+1)i\right)
	+ a_k^{i+2}\left((i+2)(i+1)\right).
	\end{eqnarray*}
	According to \eqref{aki} the coefficient $a_k^i$ satisfies
	\begin{eqnarray}
		a_k^{i+1}&=& a_k^{i}+a_{k-1}^{i+1}, \text{ and} \label{39}\\
		a_{k-1}^{i+1}&=& a_k^{i}\cdot \frac{k}{i+1}\cdot \frac{k-i-2}{k-i}, \text{ for } k>i. \label{40}
	\end{eqnarray}
	We use Equation \eqref{39} to rewrite the expression of $c_k^{i}$ for $i\le k-2$, and then use Equation \eqref{40} to simplify the equation in the following way
	\begin{eqnarray*}
		c_k^i&=& a_k^i\left(k(k-1)\right)+
			 a_{k-1}^{i+1}\left(-2(k-1)(i+1)\right)+ a_{k-2}^{i+2}\left((i+2)(i+1)\right)\\
			 &=& a_k^i k(k-1)-2a_k^ik(k-1)\frac{k-i-2}{k-i}
			 + a_{k}^{i}k(k-1)\frac{k-i-4}{k-i}\\
			 &=& 0.
	\end{eqnarray*}
	Noticing that $a_k^k=0$ shows that $f_k(q)$ has a degree of $k-1$. This implies that the
	term $q^{k-1}$ has the highest degree in \eqref{deri}. Its coefficient is
	\begin{eqnarray*}
		c_k^{k-1}&=& a_k^{k-1}\left((k+1)k+2(k+1)(k-1)+(k-1)(k-2)\right)\\
			&=&2(2k-1)\binom{2k-2}{k-1}.
	\end{eqnarray*}
	This proves the claimed equality in~\eqref{ddq}.
  \end{proof}

\begin{lemma}\label{d2A}
	For all $k\in\mathbb{N}$ and all $0<p\le 1/2$, it holds that
	\begin{eqnarray*}
		\frac{\mathrm{d}^2A(2k+1,p,1-p)}{\mathrm{d}p^2}&=&c_k p^{k-1}(1-p)^{k-1},\\
		\frac{\mathrm{d}A(2k+1,p,1-p)}{\mathrm{d}p}&=&c_k
		\int_0^p x^{k-1}(1-x)^{k-1}\mathrm{d}x \text{ and } \\
		 A(2k+1,p,1-p)&=&c_k\int_{0}^{p}\int_0^y x^{k-1}(1-x)^{k-1}\mathrm{d}x \mathrm{d}y.
	\end{eqnarray*}
	Furthermore for all $k\in\N_0$ we can write  
	\[
	A(2k+1,p,1-p)=\int_0^p \frac{\mathrm{d}A(2k+1,x,1-x)}{\mathrm{d} p}\mathrm{d}x.
	\]	
\end{lemma}

\begin{proof}
	The first equality can be easily obtained from the equality in~\eqref{ddq}. 
	Consequently, 
	\begin{gather*}
	\frac{\mathrm{d}A(2k+1,p,q)}{\mathrm{d}p}=c_k
		\int_0^p x^{k-1}(1-x)^{k-1}\mathrm{d}x +C_1\text{ with } C_1\in\R,\\
	A(2k+1,p,q)=c_k\int_{0}^{p}\int_0^y x^{k-1}(1-x)^{k-1}\mathrm{d}x \mathrm{d}y +C_1p+C_2 \text{ with } C_2\in\R.
	\end{gather*}
	Using the fact that $\lim_{p\to 0} A(2k+1,p,q)=o(p)$ for $k\ge 1$, we obtain $C_1=C_2=0$ as claimed.
	
	For the last statement, we only need to consider the case $k=0$. 
	Recalling that $A(1,p,1-p)=p$ and $\mathrm{d}A(1,p,1-p)/\mathrm{d}p=1$ shows that the equality also applies to this case.
  \end{proof}	

We next prove Lemma~\ref{LEM:CUTOFF}. 
 
\begin{proof}[Proof of Lemma~\ref{LEM:CUTOFF}]
	Using the notation from Lemma~\ref{lemma:ddq}, we first notice that for all $k>0$ we have 
	\begin{equation}\label{c_k}
	\frac{c_{k+1}}{c_k}
	= \frac{2(2k+1)(2k+3)\binom{2k}{k}}{2(2k-1)(2k+1)\binom{2k-2}{k-1}}
	=\frac{4k+6}{k}
	>4.
	\end{equation}
	Let $0<k_1<k_2$. By the above, we have $4<(4+6/k_2)^{k_2-k_1}<c_{k_2}/c_{k_1}<(4+6/k_1)^{k_2-k_1}$. Notice that $c_{k_1} (pq)^{k_1-1} - c_{k_2} (pq)^{k_2-1} =(pq)^{k_1-1}(c_{k_1}-c_{k_2}(pq)^{k_2-k_1})$.
	We now use the fact that $\lim_{p\to 0}pq=0 $ and $\lim_{p\to 1/2}pq=1/4$ to obtain that for all $k_2>k_1>0$,
	\begin{eqnarray*}
		&\lim_{p\to 0}(c_{k_1}-c_{k_2}(pq)^{k_2-k_1})>0 \text {  while  } \lim_{p\to 1/2} \left( c_{k_1}-c_{k_2}(pq)^{k_2-k_1}\right)<0.
	\end{eqnarray*}
	
	 This shows that the function $p \mapsto c_{k_1} (pq)^{k_1-1}$ intersects with $p \mapsto c_{k_2} (pq)^{k_2-1}$ in at most one point $p_I\in(0,1/2)$.
	 Moreover, we have that $c_{k_1} (pq)^{k_1-1} \ge c_{k_2} (pq)^{k_2-1}$ if and only if $p\in[0,p_I]$.
	Therefore, when $p>0$, the function $p \mapsto \int_0^p c_{k_1}(x-x^2)^{k_1-1}\mathrm{d}x$ intersects with the function $p \mapsto \int_0^p c_{k_2}(x-x^2)^{k_2-1}\mathrm{d}x$ at most once for $p\in(0,1/2)$. 
	We now prove that the intersection exists.  
	
	Notice that for all $k>1$ we have
	\begin{eqnarray*}
	 c_{k}\int_0^{0.5}(x-x^2)^{k-1}\mathrm{d}x=\frac{c_{k}}{2}\int_0^{1} (x-x^2)^{k-1}\mathrm{d}x=\frac{c_{k}}{2}\beta(k,k)=2k+1,
	\end{eqnarray*}
	and thus 
	\begin{eqnarray*}
		\lim_{p\to 1/2} \left(\int_0^p c_{k_1}(x-x^2)^{k_1-1}\mathrm{d}x - \int_0^p c_{k_2}(x-x^2)^{k_2-1}\mathrm{d}x \right)=2(k_1-k_2)<0,
	\end{eqnarray*}
	while 
	\begin{eqnarray*}
		\int_0^p c_{k_1}(x-x^2)^{k_1-1}\mathrm{d}x > \int_0^p c_{k_2}(x-x^2)^{k_2-1}\mathrm{d}x \text{ for all $p\in(0,p_I)$}.
	\end{eqnarray*}
	There exists a intersection point $p_{II}\in(0,1/2)$ such that 
	 \begin{eqnarray*}
	 	\int_0^p c_{k_1}(x-x^2)^{k_1-1}\mathrm{d}x \ge \int_0^p c_{k_2}(x-x^2)^{k_2-1}\mathrm{d}x \text{ if and only if } p\in[0,p_{II}].
	 \end{eqnarray*}
	This shows that for all $k_2>k_1> 0$ there exists a point $p_{II}\in(0,1/2)$ such that 
    \begin{eqnarray}\label{deri2}
    	\frac{\mathrm{d}A(2k_1+1,p,q)}{\mathrm{d}p} \ge \frac{\mathrm{d}A(2k_2+1,p,q)}{\mathrm{d}p} \text{ if and only if }  p\in[0,p_{II}].
    \end{eqnarray}
    To extend the conclusion to $k_1=0$, let $k_2>k_1=0$. We have $\lim_{p\to 1/2} \int_0^p c_{k_2}(x-x^2)^{k_2-1}dx=2k_2+1$ and $\lim_{p\to 0} \int_0^p c_{k_2}(x-x^2)^{k_2-1}dx=0$, while $\mathrm{d}A(1,p,q)/\mathrm{d}p=1$. Therefore the intersection point $p_{II}$ still exists and there is a unique such point.
    
	As a result we see that the function $\int_0^p \mathrm{d}A(2k_1+1,x,1-x)$ intersects with the function $\int_0^p \mathrm{d}A(2k_2+1,x,1-x)$ at most once for $p\in(0,1/2)$ and
	\begin{eqnarray*}
		&&\int_0^p \mathrm{d}A(2k_1+1,x,q) > \int_0^p \mathrm{d}A(2k_2+1,x,q) \text{ for all $p\in(0,p_{II})$ while}\\
		&&\lim_{p\to 1/2} \left(A(2k_1+1,p,q) - A(2k_2+1,p,q)\right)<0.
	\end{eqnarray*}
	This shows that $A(2k_1+1,p,q)$ intersects with $A(2k_2+1,p,q)$ exactly once at some value $p_{III}<1/2$.   
    
  \end{proof}

We are now ready to prove the monotonicity of $\Rappe$. 
\begin{proof}[Proof of the first part of Theorem~\ref{THM:MONOTONICITY}]
Let $\eps>0$, let $p_0 \in (0,1)$, and set $q_0:=1-p_0$.
	By Lemma~\ref{A2k} it holds that $A(2k,p_0,q_0)<A(2k+1,p_0,q_0)$. This shows that $\Rappe(p_0)$ is odd. Let $k \in \N \cup \{0\}$ such that $\Rappe(p_0)=2k+1$. By definition of $\Rappe$ (cf.\ equation~\eqref{def_Rapp}), $k$ is the smallest integer obtaining a drift of $A(2k+1,p_0,q_0)$. For all integers $k'<k$ we thus obtain
	\begin{equation}
	 A(2k+1,p_0,q_0)>A(2k'+1,p_0,q_0).
	\end{equation}
	By Lemma~\ref{cmp} we also get that for all $p>p_0$ it holds that 
	\begin{equation}
	A(2k+1,p,q)>A(2k'+1,p,q).
	\end{equation}
	Therefore $\Rappe(p)\ge 2k+1$ for all $p>p_0$. Since the statement holds for all $p_0\in(0,1/2-\eps]$ we obtain the monotonicity of $\Rappe$.
  \end{proof}

\subsubsection{\texorpdfstring{$\Rappe(p)=1$}{R=1} when \texorpdfstring{$0<p<1/3$}{0<p<1/3} and \texorpdfstring{$R_{\opt}(d,n)=1$}{R(d,n)=1} when \texorpdfstring{$0<d=o(n)$}{0<d=o(n)} }

	We first show that flipping one bit is optimal for the approximated drift when the distance to the optimal solution is less than $n/3$.
	
	\begin{lemma}\label{lem:R13}
		For all $\eps>0$ and all $0<p<1/3$ it holds that $\Rappe(p)=1$. 
	\end{lemma}
	
	\begin{proof} 
		Let $\eps>0$. Due to the monotonicity of $\Rappe$, it suffices to show that $\Rappe(1/3)=1$. By Lemma~\ref{A2k} we only need to consider odd values of~$r$. According to Lemma~\ref{d2A} if the second derivative $c_kx^{k-1}(1-x)^{k-1}>c_{k+1}x^k(1-x)^k$ for all $x\in(0,1/3)$ then $A(2k-1,1/3,2/3)>A(2k+1,1/3,2/3)$ and thus $\Rappe \neq 2k+1$. We notice that
		\[
		\frac{c_{k+1}~ x^k (1-x)^k}{c_{k}~ x^{k-1} (1-x)^{k-1}} = \frac{4k+6}{k}x(1-x)\le \frac{4k+6}{k}\cdot\frac{2}{9} \text{ for all } 0<x<\frac{1}{3}.
		\]
		For all $k>12$ it holds that $(4k+6)/k<9/2$, which implies that $c_kx^{k-1}(1-x)^{k-1}>c_{k+1}x^k(1-x)^k$. We therefore obtain that $\Rappe(1/3)\le 25$. For the remaining values, i.e., for $r=1,3,\ldots,25$, we can compute $A(r,1/3,2/3)$ numerically. This numerical evaluation shows that the maximum value $1/3$ is obtained (only) by $r=1$ and $r=3$. This proves $\Rappe(1/3)=1$.
	  \end{proof}
	
	We show in Lemma~\ref{lem:R_dis=1} that flipping one bit is also optimal for the exact fitness drift when the distance is a lower-order term of $n$.
	
	\begin{lemma}\label{lem:R_dis=1}
	For all $0<d=o(n)$ it holds that $R_{\opt}(d,n)=1$. 
	\end{lemma}

\begin{proof} 
	Since $d<n/4$, Lemma~\ref{lem:constantr} yields with $\eps:=1/4$ that 
	$R_{\opt}(d,n)<4^4\log(4)$. Referring to Lemma~\ref{lem:HilfeBA}, we obtain for all $3\le r\le 4^4\log(4)$ that $B(n,d,r)<A(r,d/n,1)=\Theta((d/n)^{(r+1)/2})=o(d/n)$. Since $R_{\opt}$ attains only odd values, we obtain $R_{\opt}(d,n)=1$ for all $0<d=o(n)$.
	  \end{proof}

	\subsection{Runtime Loss From Using the Approximated Drift} 
	\label{sub:runtime error}
		
	We show in this section that the expected runtimes of the exact and the approximate drift maximizer do not differ substantially. More precisely, we show that also the approximate drift maximizer also obtains an expected runtime on \onemax that is very close to that of an optimal unary unbiased black-box algorithm, cf. Corollary~\ref{cor:runtimediff}. To make things precise, we denote for every $\eps>0$ by $\tilde{A}_{\eps}^*$ the algorithm which we obtain from Algorithm~\ref{alg:algo} by replacing the mutation rate $R(\OM(x))$ by $\Rappe(1-\OM(x)/n)$. 
	
	 To state the main result, for all $\eps>0$, for all $n \in \N$, and all $0< p\le 1/2-\eps$ we abbreviate 
	 \begin{equation}\label{maxb}
	 A_{\max,\eps}(p):=A(\Rappe(p),p,1-p) \text{ and } 
	 B_{\max}(p,n):= B(n,\lfloor pn \rfloor,R_{\opt}(\lfloor pn \rfloor,n)).
	 \end{equation}
	 We notice from Lemma~\ref{lem:constantr} and Lemma~\ref{lem:constantr_approx} that $\Rappe(1/2-\eps)=\Theta(1)$ and $R_{\opt}(\lfloor(1/2-\eps)n\rfloor,n)=\Theta(1)$. Considering the drift of single bit flip, we see that $A_{\max,\eps}(1/2-\eps)\ge 1/2-\eps$. According to the definition of $\tilde{h}$ in equation \eqref{def:h_tilde}, we have $B_{\max}(d/n,n)=\tilde{h}(d)\ge d/n$ for all $0<d\le (1/2-\eps)n$.
	  
	 
	
	\begin{theorem}\label{thm:approx_h}
	For all constant $0<\eps<1/2$ the expected runtime of algorithm $\tilde{A}_{\eps}^*$ on \OneMax satisfies 
		\begin{equation*}
		E\left(T_{\tilde{A}_{\eps}^*}\right)\le \sum_{x=1}^{(1/2-\eps)n}\frac{1}{h(x)}+\frac{\eps n}{A_{\max,\eps}(1/2-\eps)}+o(n).
		\end{equation*}	
		Moreover,
		\[
		E\left(T_{\tilde{A}_{\eps}^*}\right)\le \sum_{x=1}^{(1/2-\eps)n}\frac{1}{A_{\max,\eps}(x/n)}+\frac{\eps n}{A_{\max,\eps}(1/2-\eps)}+o(n).
		\]
	\end{theorem}	
	
	\begin{proof}
	Let constant $\eps$ be a constant with $0<\eps<1/2$. It is easily seen from Theorem~\ref{dis_U} and from the definition of $\Rappe$ in~\eqref{def_Rapp} that 
	\begin{equation}
	E\left(T_{\tilde{A}_{\eps}^*}\mid x(0)\right)\le \sum_{x=1}^{n-\OM(x(0))}\frac{1}{B(n,x,\Rappe(x/n))}\le \sum_{x=1}^{n/2}\frac{1}{B(n,x,\Rappe(x/n))}+1.
	\end{equation}
	To ease representation, let $r_A(x):=\Rappe(x/n)$ and $r_B(x):=R_{\opt}(x,n)$ for all $0<x\le (1/2-\eps)n$.	
	According to Theorem~\ref{THM:APPROXBA} we have 
	\[
	\left|A\left(r_A(x),\frac{x}{n},1-\frac{x}{n}\right)-B\left(n,x,r_A(x)\right)\right|=O(1/x) \text{ and }\left|A\left(r_B(x),\frac{x}{n},1-\frac{x}{n}\right)-B(n,x,r_B(x))\right|=O(1/x).
	\]
	Since $A(r_B(x),x/n,1-x/n)\le A(r_A(x),x/n,(n-x)/n)=A_{\max,\eps}(x/n)$ and $B(n,x,r_A(x))\le B(n,x,r_B(x))=B_{\max}(x/n,n)=\tilde{h}(x)$, we obtain from Theorem~\ref{THM:APPROXBA} that, for all $0<x<(1/2-\eps)n$,
	\begin{align*}
	\left|B(n,x,r_A(x))-\tilde{h}(x)\right|&=O(1/x) \text{ and }\\ 
	|A_{\max,\eps}(x/n)-B(n,x,r_A(x))|&=O(1/x),
	\end{align*}
	where we use the fact that $r_A(x)=\Theta(1)$ according to Lemma~\ref{lem:constantr_approx}. 
	Referring to Lemma~\ref{lem:R13} and Lemma~\ref{lem:R_dis=1}, we have $r_B(x)=r_A(x)=1$ when $0<x=o(n)$. Therefore,
	\[
	A_{\max,\eps}(x/n)=B(n,x,r_A(x))=\tilde{h}(x) \text{ for all } 0<x=o(n).
	\]
	Notice that $h(x)\ge B(n,x,1)=x/n$ for all $x>0$. Using the fact that $0<h(x)-\tilde{h}(x)\le n\exp(-\Omega(n^{0.2}))$ for $0<x<n/2-n^{0.6}$ in Lemma~\ref{lem:compare_h}, we obtain
	\begin{eqnarray*}
	\left| \frac{1}{B(n,x,r_A(x))}-\frac{1}{h(x)}\right|= \frac{|h(x)-B(n,x,r_A(x))|}{B(n,x,r_A(x))h(x)}\le  
	\begin{cases}
		&	o(1/n) \text{ for } 0<x\le n^{0.99},\\
		&  O((1/x)/(x/n)^2) \text{ for } n^{0.99}<x\le (1/2-\eps)n.
	\end{cases}.
	\end{eqnarray*}
	Referring to Lemma~\ref{lem:mono_h} and the definition of $B(n,x,r)$  we see that, for all fixed $n$ and $r$, the fitness drift $B(n,x,r)$ monotonically increases with respect to $x$. Therefore, for all $x>(1/2-\eps)n$, we have $r_A(x)=r_A((1/2-\eps)n)$ and $B(n,x,r_A(x))\ge B(n,(1/2-\eps)n,r_A((1/2-\eps)n))\ge A_{\max,\eps}(1/2-\eps)-O(1/n)\ge 1/2-\eps-O(1/n)=\Omega(1)$, thus
	$$\E(T_{\tilde{A}_{\eps}^*})\le \sum_{x=1}^{(1/2-\eps)n}\frac{1}{h(x)}+O(1)+O(n^{0.03})+\frac{\eps n}{A_{\max,\eps}(1/2-\eps)}+o(n).$$ 
	This proves the first statement.
	
	The second statement can be shown by using similar methods to bound the absolute difference $|1/B(n,x,r_A(x))-1/A_{\max,\eps}(x/n)|$ for $0<x\le n/2$.
	 \end{proof}
	
	\begin{corollary}\label{cor:runtimediff}
	For all constant $\eps$ with $0<\eps<1/2$, the difference between the expected runtime of $\tilde{A}_{\eps}^*$ on \onemax and that of an optimal unary unbiased black-box algorithm is $O(\epsilon n)$.	Furthermore, the absolute difference between the expected runtimes of $A^*$ and $\tilde{A}_{\eps}^*$ is also $O(\epsilon n)$. 
	\end{corollary}
	
	\begin{proof} Let constant $0<\eps<1/2$ and let $A$ be an arbitrary unary unbiased black-box algorithm. By Theorems~\ref{thm:LBany} and~\ref{thm:approx_h} it holds that 
	\begin{align*}
				E\left(T_{A}\right) 
	& \ge \sum_{x=1}^{n/2-n^{0.6}}\frac{1}{h(x)}-\Theta(n^{2/3}\ln^9(n))\\
	& \ge \sum_{x=1}^{(1/2-\eps)n}\frac{1}{h(x)}-o(n)\\
	& \ge E\left(T_{\tilde{A}_{\eps}^*}\right) - O(\eps n).
	\end{align*}
	
	The second statement is a direct consequence of the first and Theorem~\ref{thm:main22}.
	 \end{proof}

\section{Runtime Analysis for the Approximate Drift-Maximizer \texorpdfstring{$\tilde{A}_{\eps}^*$}{}} 

We compute in this section the expected time needed by Algorithm $\tilde{A}_{\eps}^*$ to optimize \onemax. We fix $0<\eps<1/2$.  
As proven in Lemma~\ref{lem:R13} algorithm $\tilde{A}_{\eps}^*$ flips $\Rappe(p)=1$ bit whenever $0<p<1/3$. In this regime $\tilde{A}_{\eps}^*$ is thus equal to RLS. It is well known (and easy to prove by a simple fitness-level argument) that the expected time needed by RLS starting in a search point of \onemax value $2n/3$ to reach the all-ones string equals $n\sum_{i=1}^{n/3}{1/i}=nH_{n/3}=n(\ln(n/3)+\gamma)+3/2+O(1/n)$, where $\gamma\approx 0.57721\dots$ denotes again the Euler–Mascheroni constant. It therefore remains to compute the time needed by $\tilde{A}_{\eps}^*$ to reach for the first time a search point having fitness at least $2n/3$. 

Formally, we also need to show that the first search point having fitness at least $2n/3$ does not have a fitness value that is much larger than this. Since we flip a constant number of bits only, we get this statement for free. Note also that it is shown below that in the interval before reaching this fitness level the algorithm flips only $3$ bits. Apart from this situation around fitness layer $2n/3$ we do not have to take care of jumping several fitness layers by hand, but this is taken into account already in the drift theorems from which we derive our runtime estimates.

\subsection{Drift Analysis}
As we did in Section~\ref{sec:maxDrift}, we employ the variable drift theorems, Theorems~\ref{dis_U} and~\ref{dis_L'}, to compute upper and lower bounds for the expected runtime of algorithm $\tilde{A}_{\eps}^*$. We will provide a numerical evaluation of these expressions in Section~\ref{sec:numerics}.


\textbf{Lower bound.} We first compute a lower bound for the expected runtime $\E(T_{A})$ of any unary unbiased black-box algorithm $A$ on \onemax. According to Theorem~\ref{thm:LBany} and using a similar method to estimate $|1/A_{\max,\eps}(x/n)-1/h(x)|=o(1)$ as in Theorem~\ref{thm:approx_h} for $0<x\le (1/2-\eps)n$, we obtain that
\begin{eqnarray*}
	\E\left(T_{A}\right)&\ge & \sum_{x=1}^{n/2-n^{0.6}}\frac{1}{h(x)}-\Theta(n^{2/3}\ln^9(n))
	\ge \sum_{x=1}^{(1/2-\eps)n}\frac{1}{A_{\max}(x/n)}-o(n)\\
	&=&  n H_{n/3}+\sum_{x=\lfloor n/3\rfloor}^{(1/2-\eps)n}\frac{1}{A_{\max}(x/n)}-o(n).
\end{eqnarray*}
Let $k \in \N$ and let $1/2-\eps=:p_0>p_1>\ldots>p_k>1/3$. Using the fact that $A_{\max,\eps}$ is increasing, we bound $\E(T_A)$ by 
\begin{equation}\label{sumLmain}
\E\left(T_{A}\right)\ge n \left(\ln\left(\frac{n}{3}\right)+\gamma + \sum_{i=1}^{k} \frac{p_{i-1}-p_{i}}{A_{\max,\eps}(p_{i-1})} +\int_{1/3}^{p_k}\frac{\mathrm{d} p}{A_{\max,\eps}(p)}\right)-o(n).
\end{equation}

\textbf{Upper bound.} Using the fact that $\Rappe(x/n)=1$ for all $0<x\le n/3$ and referring to Theorem~\ref{thm:approx_h}, we obtain 
\begin{eqnarray*}
	E\left(T_{\tilde{A}_{\eps}^*}\right)&\le& nH_{n/3}+ \sum_{x=\lfloor n/3\rfloor}^{(1/2-\eps)n}\frac{1}{A_{\max,\eps}(x/n)}+\frac{\eps n}{A_{\max,\eps}(1/2-\eps)}+o(n).
\end{eqnarray*}
Using the same partition points as in the lower bound statement and the monotonicity of $A_{\max,\eps}$, we have
\begin{equation}\label{sumUmain}
\E(T_{\tilde{A}_{\eps}^*})\le n \left(\ln\left(\frac{n}{3}\right)+\gamma + \sum_{i=1}^{k} \frac{p_{i-1}-p_{i}}{A_{\max,\eps}(p_{i})} +\frac{1/2-p_{0}}{A_{\max,\eps}(p_{0})} +\int_{1/3}^{p_k}\frac{\mathrm{d} p}{A_{\max,\eps}(p)}\right)+o(n).
\end{equation}

\subsection{Numerical Evaluation of the Expected Runtime}
\label{sec:numerics}

In this section we evaluate expressions~\eqref{sumLmain} and~\eqref{sumUmain} numerically to compute an estimate for the expected runtime of algorithm $\tilde{A}_{\eps}^*$ on \onemax and for the unary unbiased black-box complexity. 

\begin{theorem}\label{THM:RUNTIME}
For sufficiently small $\eps>0$ the expected runtime $\E(T_{\tilde{A}_{\eps}^*})$ of algorithm $\tilde{A}_{\eps}^*$ on \onemax is $n \ln(n) - cn \pm o(n)$ for a constant $c$ between $0.2539$ and $0.2665$. This bound is also the unary unbiased black-box complexity of \onemax. 
\end{theorem} 

We can rewrite the expression in Theorem~\ref{THM:RUNTIME} to $n \left(\ln\left(n/3\right)+\gamma+c'\right) +o(n)$ for a constant $c'$ between $0.2549$ and $0.2675$ to ease a comparison with the expected runtime of the previously best known unary unbiased algorithm, which is the one presented in~\cite{AxelDD15}. This latter algorithm has an expected runtime equaling that of RLS up to an additive term of order $o(n)$. It is hence $n(\ln(n/2)+\gamma)\pm o(n)$. For sufficiently small $\eps>0$ Algorithm $\tilde{A}_{\eps}^*$ is thus by an additive $(\ln(3)-\ln(2)-c') n\pm o(n)$ term faster, on average, than RLS or the algorithm presented and analyzed in~\cite{AxelDD15}. That is, compared to RLS, algorithm $\tilde{A}_{\eps}^*$ saves between $0.138 n \pm o(n)$ and $0.151 n \pm o(n)$ iterations on average.

To compute $\E(T_{\tilde{A}_{\eps}^*})$, we split the interval $(0,\frac{1}{2})$ into intervals $(L_{2i+1},R_{2i+1}]$, $i=0,1,\ldots$, such that for each $i$ and each $p\in (L_{2i+1},R_{2i+1}]$ the number $\Rappe(p)$ of bits that need to be flipped in order to maximize the approximated expected fitness increase $A(\cdot,p,1-p)$ is $2i+1$ (note that this is independent of $\eps$, since $\eps$ just determines the cut-off point after which only use a bound for the drift-maximizing number of bit flips). 
Table~\ref{tab:intervals} displays the first few intervals along with the corresponding drift values at the borders of the interval. We observe that the further we are away from the optimum (this corresponds to larger $r$ by Theorem~\ref{thm:monotonicity}), the smaller the size of the interval. 

\begin{table*}[t]
\begin{center}
\begin{scriptsize}
\begin{tabular}{cccccc}
\centering
$r$ & $L_r$ & $R_r$ & $A_{\max,\eps}(L_r)$ & $A_{\max,\eps}(R_r)$ & $R_r-L_r$\\
\hline
3  & 0.333333333 & 0.367544468 & 0.333333 & 0.405267 & 0.034211135 \\
5  & 0.367544468 & 0.386916541 & 0.405267 & 0.467174 & 0.019372073 \\
7  & 0.386916541 & 0.399734261 & 0.467174 & 0.522084 & 0.012817721 \\
9  & 0.399734261 & 0.409006003 & 0.522084 & 0.571870 & 0.009271741 \\
11 & 0.409006003 & 0.416109983 & 0.571870 & 0.617718 & 0.007103980
\end{tabular}
\caption{The optimal number of bit flips in interval $(L_{r}n,R_{r}n]$ is $r$.}
\label{tab:intervals}
\end{scriptsize}
\end{center}
\end{table*}
The bound for the expected runtime of algorithm $\tilde{A}_{\eps}^*$ reported in Theorem~\ref{THM:RUNTIME} is obtained by setting $p_0=R_{4001}$, $p_k=R_9$ and using the following partition points
\begin{eqnarray*}
	p_0=R_{4001}>R_{3001}>R_{2001}>R_{1001}>R_{951}>R_{901}>R_{851}>\\
	\cdots>R_{200}>R_{151}>R_{101}>R_{35}>R_{34}>\cdots>R_{10}>R_{9}=p_k.
\end{eqnarray*}
The accuracy of our approximation can be increased by adding denser partition points, especially to the smaller side near $p_k$.

\section{Fixed-Budget Analysis}
\label{sec:budget}

In this section, we compare the algorithms developed in this work with the classic RLS heuristic (which was the essentially best previous unary unbiased algorithm for $\OM$) in the \emph{fixed-budget perspective}, that is, we compare the expected fitnesses obtained after a fixed budget $B$ of iterations. This performance measure was introduced by Jansen and Zarges~\cite{JansenZ14} to reflect the fact that the most common use of search heuristics is not to compute an optimal solution, but only a solution of reasonable quality. We note that the time to reach a particular solution quality, called $T_{A,f}(a)$ in~\cite{DoerrJWZ13gecco} where this notion was first explicitly defined, would be an alternative way to phrase such results. We do not regard this performance measure here, but we would expect that, in a similar vein in as the following analysis, also in this measure our algorithm is superior to RLS by a (small) constant percentage.

Our main result in this section is that our drift maximizer with a fixed budget compute solutions having a roughly $13$\% smaller fitness distance to the optimum. This result contrasts the lower-order advantage in terms of the expected runtime, i.e., the average time needed to find an optimal solution. 

The main challenge is proving the innocent statement that the time taken by our algorithm to find a solution $x$ of fitness $\OM(x) \ge 2n/3$ is strongly concentrated. Such difficulties occur often in fixed-budget analyses, see, e.g.~\cite{DoerrJWZ13gecco}. We prove the desired concentration via the following well-known martingale version of Azuma's inequality~\cite{Azuma67} (as opposed to the simpler method of bounded differences, which appears not to be applicable here). 

\begin{theorem}[Method of Bounded Martingale Differences] \label{thm:martingale}
	Let $X_1,X_2,\dots,X_n$ be an arbitrary sequence of random variables and let $f$ be a function satisfying the property that for each $i\in[n]$, there is a non-negative $c_i$ such that $|\E(f\mid X_0,X_1,\dots,X_{i-1})-\E(f\mid X_0,X_1,\dots,X_i)|\le c_i$.
	Then 
	\begin{equation*}
		\Pr\left(|f-\E(f)|\ge \delta\right)\le 2 \exp\left(-\frac{\delta^2}{2\sum_{i=1}^{n}c_i^2}\right)
	\end{equation*}
	for all $\delta>0$.
\end{theorem}

Consider a run of the algorithm $\tilde{A}_{\eps}^*$ with small constant $0<\eps<1/6$. Let $X_t:=n-\max\{\OM(x(i)) \mid i \in [0..t]\}$ 
be the current smallest fitness distance and let $r_{\max}:=\Rappe(1/2-\eps)$ be the  maximal mutation strength. Let $T_{1/3}$ be the first time at which the distance to the optimum is at most $n/3$, i.e., $T_{1/3}$ is the smallest $t$ for which $X_t\le n/3$. 
Let $N:=3r_{\max}n$ and define the function $f$ by setting $f(X_0,X_1,\dots):=\min\{N,T_{1/3}\}$.

We notice that $\Pr(T_{1/3}>f)=\Pr(T_{1/3}> N)= \Pr(X_{N}>n/3 )$. Referring  to Lemma~\ref{THM:APPROXBA}, we obtain for all $X_t> n/3$ that $\E(X_t-X_{t+1}\mid X_t)=B(n,X_t,\Rapp(X_t/n))\ge A_{\max,\eps}(X_t/n)-O(1/X_t)\ge 1/3-o(1)$. Using the fact that $X_t-X_{t+1}\le r_{\max}$, we obtain $\Pr(X_t>X_{t+1}\mid X_t>n/3)\ge 1/(3r_{\max}) - o(1)$. Define binary random variables $Y_t$ by setting $Y_t:=\mathbbm{1}_{X_t>X_{t+1}}$, if $X_t > n/3$, and otherwise by having $Y_t = 1$ with probability $1/(3r_{\max}) - o(1)$ independently for all such $Y_t$. Note that, by definition, we have $\Pr[Y_t = 1] \ge 1/(3r_{\max}) - o(1)$ regardless of the outcomes of $Y_{t'}$, $t' < t$. Consequently, by well-known results, e.g., Lemma~3 in~\cite{Doerr18evocop}, the $Y_t$ admit the same Chernoff bounds for the lower tail as independent binary random variables with success probability $1/(3r_{\max}) - o(1)$. We thus estimate $\Pr(X_N>n/3)\le\Pr(Y_1+Y_2+\dots+Y_N<(2/3)n)=\exp(-\Omega(n))$.
Consequently $E(T_{1/3}-f)<E(T)\Pr(T_{1/3}>f)=\exp(-\Omega(n))$.


Using the fact that $X_t-X_{t+1}\le r_{\max}$, the additive drift theorem yields that the expected influence of one iteration on the remaining optimization time is at most $r_{\max}/(1/3-o(1))<4r_{\max}$. Consequently, for $1\le i\le N$, we have
\begin{equation*}
	|\E(f\mid X_0,X_1,\cdots,X_{i-1})-\E(f\mid X_0,X_1,\cdots,X_i)|\le 4r_{\max}.
\end{equation*}
Applying Theorem~\ref{thm:martingale} and using the fact that $\Pr(T_{1/3}>f)=\exp(-\Omega(n))$ and $E(T_{1/3}-f)=\exp(-\Omega(n))$, we compute
\begin{eqnarray*}
	\Pr\left(|T_{1/3}-\E(T_{1/3})|\ge n^{0.6}\right)&\le&
	\Pr\left(|f-\E(T_{1/3})|\ge n^{0.6}\right)+\exp(-\Omega(n))\\
	&\le& \Pr\left(|f-\E(f)-\E(T_{1/3}-f)|\ge n^{0.6}\right)+\exp(-\Omega(n))\\
	&\le& \Pr\left(|f-\E(f)|\ge n^{0.6}-\E(T_{1/3}-f)\right)+\exp(-\Omega(n))\\
	&\le&\Pr\left(|f-\E(f)|\ge n^{0.6}/2\right)+\exp(-\Omega(n))\\
	&\le& 2\exp\left(-\frac{n^{1.2}/4}{2N(4r_{\max})^2}\right)+\exp(-\Omega(n))=o(\exp(-n^{0.1})).
\end{eqnarray*}
According to the computation in the proof of Theorem~\ref{THM:RUNTIME} and using the fact that $\E(T_{1/3})=\E(T_{\tilde{A}_{\eps}^*})-nH_{n/3}$, we obtain with probability $1-O(\exp(-n^{0.1}))$ that $0.2549n \le T_{1/3} \le 0.2675n$. 

Consider a budget of $B=kn$ iterations with $k \ge 0.2675$. Let $s:= \lfloor 0.2675n \rfloor$. With probability $1-O(\exp(-n^{-0.1}))$, a run of algorithm $\tilde{A}_{\eps}^*$ has $X_s \le n/3$. Conditional on this, in the remainder $\tilde{A}_{\eps}^*$ mutates exactly one bit in each iteration according to Lemma~\ref{lem:R13}. Since $E(X_{t} \mid X_{t-1})=X_{t-1}(1-1/n)$ in this case, we have for all $t \ge s$ that
\[\E(X_{t} \mid X_s) = \E(X_s)\left(1-1/n\right)^{t-s} \le (n/3)(1-1/n)^{t-s}.\]
Therefore, with a budget of $B \ge 0.2675n$ algorithm $\tilde{A}_{\eps}^*$ reaches a fitness distance $X_B$ satisfying 
\begin{eqnarray*}
  \E(X_B ) \le \E(X_{B} \mid X_s\le n/3)+ n\cdot \Pr(X_s>n/3)\le
   (1+o(1))(n/3)(1-1/n)^{B - 0.2675n}.
\end{eqnarray*}

Using the same reasoning for RLS, we compute for $Y_B$ the fitness distance RLS reaches with the same budget of $B$ that
\begin{eqnarray*}
	\E(Y_B)&=& (n/2)(1-1/n)^B\\
	&=& (3/2)(1-1/n)^{0.2675n} (n/3)(1-1/n)^{B - 0.2675n}\\
	&\ge& (1-o(1))(3/2)\exp(-0.2675) \E(X_B) \\
	&=& (1-o(1)) 1.1479...\, \E(X_B).
\end{eqnarray*}
In other words, $\E(X_B) \le (1+o(1)) 0.8711...\, \E(X_A)$, that is, with the same budget, Algorithm $\tilde{A}_{\eps}^*$ is roughly $13\%$ closer to the optimum than RLS.


\paragraph{Acknowledgments}
This research benefited from the support 
of the ``FMJH Program Gaspard Monge in optimization and operation research'', 
and from the support to this program from \'Electricit\'e de France. It has also been supported by a public grant as part of the
Investissement d'avenir project, reference ANR-11-LABX-0056-LMH,
LabEx LMH.

}
\newcommand{\etalchar}[1]{$^{#1}$}

\end{document}